\pdfoutput=1
\documentclass{article}

\usepackage{arxiv}

\usepackage[utf8]{inputenc}
\usepackage[T1]{fontenc}
\usepackage{courier}
\usepackage{url}
\usepackage{booktabs}
\usepackage{array}
\usepackage{amsmath}
\usepackage{amssymb}
\usepackage{amsfonts}
\usepackage{amsthm}
\usepackage{nicefrac}
\usepackage{microtype}
\usepackage{graphicx}
\usepackage{algorithm}
\usepackage{algpseudocode}
\usepackage{enumitem}
\usepackage[numbers]{natbib}
\usepackage{hyperref}
\usepackage{cleveref}
\usepackage{doi}

\graphicspath{{./}{Figures/}}

\newtheoremstyle{upright}{6pt}{6pt}{\normalfont}{}{\bfseries}{.}{0.5em}{}
\theoremstyle{upright}
\newtheorem{proposition}{Proposition}
\newtheorem{lemma}{Lemma}
\newlist{props}{enumerate}{1}
\setlist[props]{label=(\alph*),leftmargin=2.2em,itemsep=2pt,topsep=3pt,parsep=0pt}

\DeclareMathOperator*{\TopK}{TopK}
\newcommand{\Pop}{\mathrm{Pop}}
\newcommand{\SCP}{\mathrm{SCP}}

\newcommand{\CSR}{\mathrm{CSR}}
\newcommand{\cF}{\mathcal{F}}
\newcommand{\cC}{\mathcal{C}}
\newcommand{\cP}{\mathcal{P}}
\newcommand{\cK}{\mathcal{K}}

\newcommand{\figw}{0.82\textwidth}
\newcommand{\figwb}{0.78\textwidth}

\title{FairFRESH: A Neuro-Symbolic Framework for Fair and Explainable
Household Food Recommendation\thanks{Accepted and presented at the 8th
International Conference on Communication and Intelligent Systems (ICCIS
2026), BITS Pilani, K K Birla Goa Campus, India; the conference version will
appear in Springer's \emph{Lecture Notes in Networks and Systems}. This arXiv
version adds the analysis of the constraint penalty and of greedy
allocation, the bargaining characterization of the allocation rules
(Section~\ref{sec:allocation}, Appendix~\ref{app:proofs}), the extension to
shared household meals and their rotation (Section~\ref{sec:shared_meals}),
and the corresponding experiments.}}

\date{}

\author{
	Kodur Sai Vinay Sathvik, Kakoli Bora, and Tejaswini R. M. \\
	Department of Artificial Intelligence and Machine Learning\\
	BNM Institute of Technology\\ \texttt{sathvik2004@gmail.com}, \texttt{kakolibora@bnmit.in}, \texttt{tejaswinirm@bnmit.in} \\
}

\renewcommand{\headeright}{}
\renewcommand{\undertitle}{}
\renewcommand{\shorttitle}{FairFRESH}

\hypersetup{
pdftitle={FairFRESH: A Neuro-Symbolic Framework for Fair and Explainable Household Food Recommendation},
pdfsubject={cs.IR, cs.AI, cs.GT},
pdfauthor={Kodur Sai Vinay Sathvik, Kakoli Bora, Tejaswini R. M.},
pdfkeywords={food recommendation, fair division, multi-agent recommendation, pantry-aware recommendation, neuro-symbolic reasoning, Nash social welfare, envy-freeness},
}

\begin{document}
\maketitle

\begin{abstract}
Recommender systems typically serve one user at a time. However in households, members share a finite pantry and follow different dietary rules, making independent recommendations prone to resource conflicts and dietary violations. We formulate household food recommendation as constrained fair division and present \textbf{FairFRESH}, a neuro-symbolic framework that learns per-member utilities, allocates foods under explicit welfare objectives, enforces dietary rules through a deterministic filter, and explains recommendations using Shapley and Banzhaf attributions. Theoretically, we show that no finite penalty guarantees compliance near a dietary-rule threshold, establish approximation guarantees for greedy allocation under utilitarian and smoothed Nash objectives in the stock-free model before filtering, and characterize the scale invariance of Nash welfare. We also use the minimax theorem to analyze rotating shared meals. On NHANES-derived data, the learned utility model achieves Recall@20 of 0.293, outperforming five learned recommenders (at most 0.244) but not a previous-day repetition heuristic (0.321). On held-out four-person households, Nash allocation reduces Gini disparity by 42\% relative to conflict resolution while improving Recall@20 from 0.180 to 0.197. It comes within 2\% of an upper bound on the smoothed Nash optimum, and rescaling one member’s utilities changes at most 1.5\% of assignments. All recommendations satisfy the encoded dietary rules, while attribution scores correlate with utility loss upon component removal ($r=0.76$). In simulation, shared meals reduce pantry use by about one-third, and weekly rotation increases the worst-off member’s utility by 38\% over the best fixed meal.
\end{abstract}


\section{Introduction}
\label{sec:intro}

Recommender systems usually model one user at a time. That assumption suits
film or music, but not food: meals are planned around a shared pantry and
shaped by the dietary needs of everyone at the table. When two members of a
household are recommended the same scarce item, at most one of them can use
it, however highly a per-user model scores it for both. When members follow
different dietary rules, the same pantry supports different safe choices for
each of them (Fig.~\ref{fig:motivation}). Household food recommendation is
therefore not a ranking problem repeated for every member but an allocation
problem over a shared, finite resource.

Existing systems address parts of this problem. Health-aware recommenders
optimize nutrition jointly with preference, as in
MOPI-HFRS~\cite{zhang2025mopi}, or reason over food knowledge graphs, as in
SwissFKG~\cite{freitas2025swiss}, but they score each user against that
user's own history. Caldeira et al.~\cite{Caldeira2018} bring pantry state
into the model, but only as a per-user ranking feature, which treats a
shared, finite resource as if it were unlimited. Group recommenders
aggregate preferences into one package that every member
shares~\cite{Masthoff2004,Serbos2017}, so they never divide a resource that
one member's choice depletes for another. Dietary safety is typically
encoded as a learned penalty or a soft objective, which a strong taste or
popularity signal can outweigh, and explanations, where present, address one
user's ranking rather than a household's allocation.

Fair division studies how to divide a finite set of goods among agents
with different preferences so that the outcome is efficient and
equitable~\cite{mishra2023fair}, which is the situation of a household
sharing one pantry. We introduce \textbf{FairFRESH}, a neuro-symbolic
framework that treats household food recommendation as fair division of
indivisible goods under per-member dietary rules.

A three-stage pipeline produces each household recommendation. In the first
stage, a learned utility model combines taste, pantry, and
retrieval-augmented dietary encoders to score every member--food pair. In
the second stage, an explicit welfare rule (utilitarian, max-min, or Nash)
allocates foods from the resulting utility matrix, so that no item is
promised to two members. In the third stage, a deterministic symbolic
filter removes any food that breaks a member's dietary rule. A post-hoc
module then explains the verified recommendations with Shapley and Banzhaf
attributions over the pipeline's components. This division of labor allows
the learned components to rank, while the filter alone decides safety and
the explanation module never feeds back into a decision. Every final
recommendation therefore satisfies the encoded dietary rules, whatever the
learned scores.

We evaluate FairFRESH on NHANES-derived households of two to six members,
comparing it with five recommender baselines and five allocation strategies.
The evaluation addresses four research questions. \textbf{RQ-A}: does a
food's availability, measured by the member's previous-day consumption,
carry predictive signal beyond taste?
\textbf{RQ-B}: how does the allocation rule change equity when members
compete for one pantry, and at what cost in accuracy? \textbf{RQ-C}: does
retrieval-augmented encoding of dietary constraints cover members'
requirements without weakening the safety guarantee? \textbf{RQ-D}: can
cooperative-game attributions explain household recommendations faithfully,
without changing the decisions they explain?

\paragraph{\textbf{Contributions}}
This paper makes the following contributions.
\begin{enumerate}
\item We introduce FairFRESH, which formulates household food
recommendation as fair division of a shared, finite pantry: servings are
private goods allocated under per-member safe sets and a cap of $K$ items
(Section~\ref{sec:problem}). We prove that no finite penalty weight can
certify safety near a rule threshold (Proposition~\ref{prop:scp}), so
FairFRESH enforces the rules with a deterministic filter, and we show that,
at the optimum, allocating within safe sets dominates filtering afterwards
(Proposition~\ref{prop:placement}).
\item The key component of FairFRESH is an explicit allocation layer. One
greedy procedure implements the utilitarian, max-min, and Nash rules; we
bound it for the utilitarian and smoothed Nash objectives before filtering
(Proposition~\ref{prop:greedy}), find it within 1--5\% of exact optima in
practice, and show through bargaining theory that only the exact Nash rule
is invariant to how members' utilities are scaled
(Proposition~\ref{prop:bargain}). Shapley and
Banzhaf attributions then explain the household allocation, not one user's
ranking (Section~\ref{sec:explainability-method}).
\item We extend FairFRESH to dishes cooked once and shared by the household.
Using the minimax theorem, we show that rotating shared dishes over days can
only help the worst-off member, and that a linear program computes the best
rotation (Proposition~\ref{prop:rotation}).
\item We provide a detailed empirical evaluation on households of two to six
members, with all allocation results on users disjoint from training. The
utility model ranks above five learned recommenders, and a repeat
heuristic shows that previous-day consumption carries most of the signal.
Under exclusive household allocation, Nash allocation lowers Gini by 42\%
relative to conflict resolution with higher Recall@20 and changes at most
1.5\% of its assignments when utilities are rescaled. Shared meals save
about a third of the pantry
(Sections~\ref{sec:results} and~\ref{sec:shared_meals}).
\end{enumerate}

\begin{figure}[t]
\centering
\includegraphics[width=\figw]{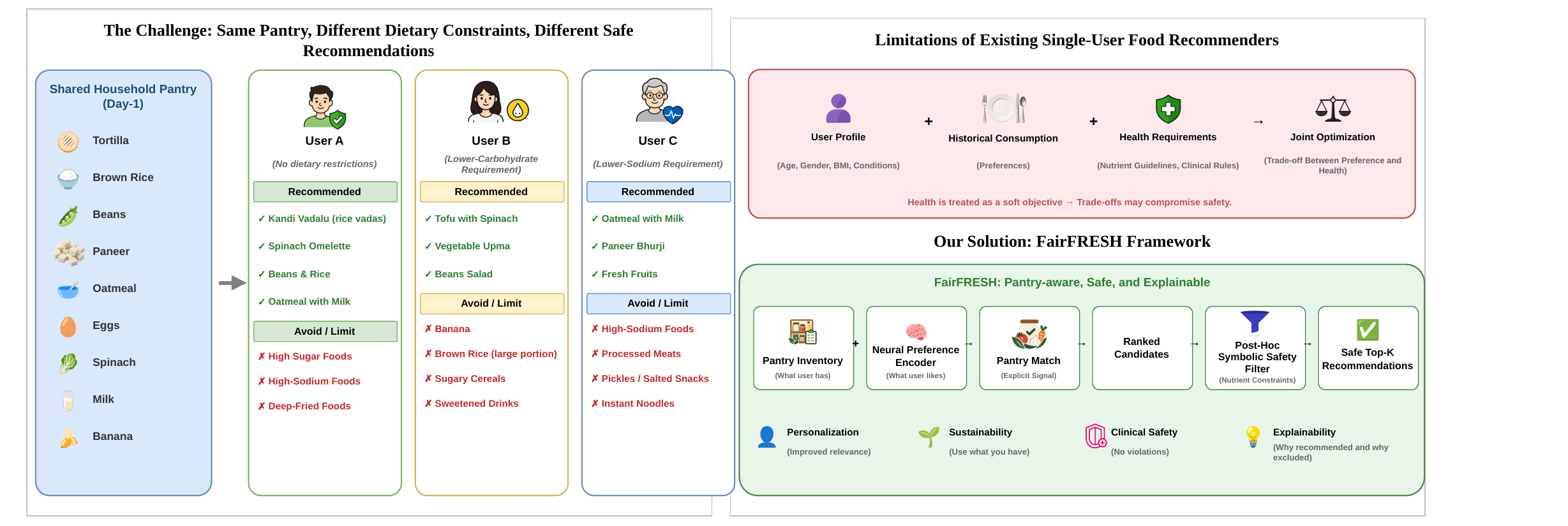}
\caption{Motivation. One shared pantry yields different safe
recommendations for members with different dietary rules, so household
recommendation is an allocation problem rather than a ranking repeated for
each member.}
\label{fig:motivation}
\end{figure}

\section{Related Work}
\label{sec:related}

Most food recommenders model one user at a time. Freyne and
Berkovsky~\cite{Freyne2010} personalize recipes from ratings, Ge et
al.~\cite{Ge2015} add a health objective, and Trattner and
Elsweiler~\cite{Trattner2017WWW} study the healthiness of online recipes.
Food knowledge graphs~\cite{Haussmann2019,Ma2024} and multimodal
input~\cite{Deng2026} encode nutritional semantics, and MOPI-HFRS
\cite{zhang2025mopi} optimizes nutrition jointly with preference. Where
ingredient availability appears~\cite{Caldeira2018}, it is a per-user
ranking feature. None of these systems divides a pantry that one member's
choice depletes for another.

Fairness in recommender systems usually concerns exposure across items or
disparities across large user groups~\cite{Burke2017}. A household is
instead a small, fully known group sharing one pool of resources. Group
recommenders aggregate member preferences~\cite{Masthoff2004} or impose
proportionality and envy-freeness on a shared package~\cite{Serbos2017}, but
every member receives the same items, so each item acts as a public good.
These systems do not model private servings drawn from a common stock.

Fair division studies exactly this setting. For indivisible goods with
additive valuations~\cite{mishra2023fair}, envy-free allocations need not
exist, but EF1 allocations do, and round robin finds one in polynomial
time~\cite{caragiannis2019unreasonable}. Without further constraints, every
Nash-optimal allocation is Pareto-optimal and
EF1~\cite{caragiannis2019unreasonable}, but computing one is
APX-hard~\cite{nguyen2014computing,lee2017apx}, and the max-min optimum is
NP-hard to approximate within any factor below
two~\cite{bezakova2005allocating}. Cardinality limits change which of these
guarantees survive~\cite{suksompong2021constraints,biswas2018fair}.

Reviewer assignment poses a closely related constrained problem, where Stelmakh et
al.\ maximize the quality of the most disadvantaged paper under load
constraints~\cite{stelmakh2021peerreview4all}; members play the role of
papers and foods that of reviewers. When every agent consumes an item,
allocation becomes a public decision judged by how well the common choice
serves each agent~\cite{conitzer2017fair,fain2018fair}. This literature
assumes known valuations and has not been connected to learned utilities or
dietary constraints.

Neuro-symbolic methods such as logic tensor
networks~\cite{Carraro2022AIxIA,Badreddine2022AI} relax symbolic constraints
into differentiable objectives, which a high taste or popularity score can
outweigh. Post-hoc explanation methods such as the Shapley
value~\cite{Lundberg2017} and the Banzhaf index~\cite{banzhaf1965weighted},
evaluated with the faithfulness correlation of Bhatt et
al.~\cite{bhatt2020evaluating}, attribute a single prediction to its inputs.
Neither line addresses a guarantee or an explanation at the level of a
household allocation.

FairFRESH takes learned per-member utilities from the recommendation
literature and passes them to a fair-division allocation layer, so the
pantry is divided rather than scored separately for each member. Its soft
penalty comes from the same symbolic rules as a deterministic filter, which
alone decides compliance (Sections~\ref{sec:utility}
and~\ref{sec:safety}). Its attributions are defined over the household
allocation rather than one user's ranking, and shared dishes enter as public
goods (Section~\ref{sec:shared_meals}).

\section{Methodology}
\label{sec:methodology}

Figure~\ref{fig:fairfresh_architecture} shows the architecture of
FairFRESH. The pipeline has four stages: (1) the utility model, (2) the
allocation layer, (3) the symbolic safety filter, and (4) the post-hoc
explainability module.

Household inputs feed the utility model, whose scores form the household
utility matrix. The allocation layer assigns foods from that matrix under
an explicit welfare rule, and the safety filter removes any food that breaks
its member's dietary rules. The explainability module then explains the
verified recommendations without feeding back into any earlier stage. We
first define the problem and then describe each stage.

\begin{figure}[ht]
\centering
\includegraphics[width=\figwb]{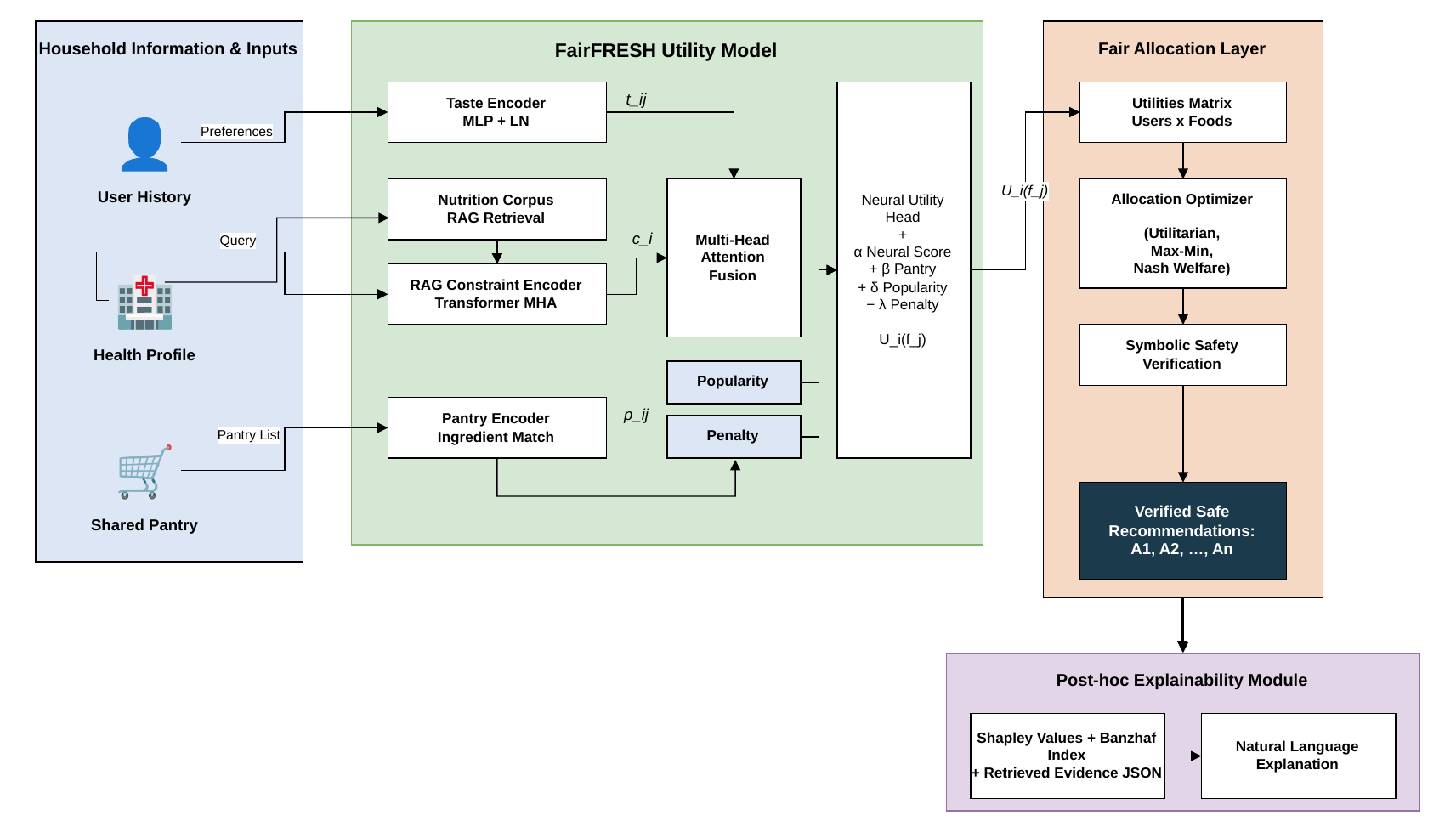}
\caption{FairFRESH. Taste, retrieval-augmented dietary, and pantry
encoders, fused by multi-head attention, produce a utility for every
member--food pair. An allocation rule (utilitarian, max-min, or Nash)
assigns foods from the utility matrix, a symbolic filter removes unsafe
foods, and a post-hoc module turns Shapley and Banzhaf attributions into
natural-language explanations without affecting any decision.}
\label{fig:fairfresh_architecture}
\end{figure}

\subsection{Preliminaries and Problem Definition}
\label{sec:problem}
\textbf{Model.} A household is a set of members $N=\{1,\ldots,n\}$ sharing a fixed pantry
$\cP$. Each member $i$ contributes stated preferences, an interaction
history, and a health profile (conditions, dietary tags, and allergies) that
determines a set of dietary rules $\cC_i$. Over a food catalog
$\cF=\{f_1,\ldots,f_m\}$, FairFRESH learns a utility $U_i(f)\in(0,1)$ for
every member--food pair (Section~\ref{sec:utility}) and recommends to each
member a bundle $A_i\subseteq\cF$. Bundles are disjoint, because an item
allocated to one member is unavailable to another, and each holds at most
$K$ items; Section~\ref{sec:shared_meals} relaxes disjointness for dishes
cooked once and shared by the household. Since $nK$ is far smaller than $m$,
allocations are partial: most of the catalog stays unallocated. Utilities
are additive, $U_i(A_i)=\sum_{f\in A_i}U_i(f)$.

\textbf{Dietary rules.} Each rule $c\in\cC_i$ is expressed as a normalized
slack $g_c(f)$, with $g_c(f)\leq 0$ meaning that $f$ satisfies the rule. For
a nutrient limit, $g_c(f)=x_{f,c}/\tau_c-1$, where $x_{f,c}$ is the amount of
the limited nutrient in a serving of $f$ and $\tau_c$ is the per-serving
limit; for an exclusion such as an allergen, $g_c(f)=1$ if $f$ contains the
excluded ingredient and $-1$ otherwise. Food $f$ is safe for member $i$ if
$g_c(f)\leq 0$ for every $c\in\cC_i$. We write $S_i$ for member $i$'s safe set
and $V_i=\cF\setminus S_i$ for its complement.

\textbf{Fairness and efficiency.} We use the standard notions of fair
division~\cite{mishra2023fair}, adapted to safe sets. A member is never
served an unsafe food, so when comparing bundles we value them by their safe
part, $U_i^{S}(B)=U_i(B\cap S_i)$. An allocation $A$ is \emph{envy-free} (EF)
if $U_i^{S}(A_i)\geq U_i^{S}(A_j)$ for all $i,j$, and \emph{envy-free up to
one good} (EF1) if, for all $i,j$ with $A_j\neq\varnothing$, some
$f\in A_j$ satisfies $U_i^{S}(A_i)\geq U_i^{S}(A_j\setminus\{f\})$. It is
\emph{Pareto-optimal} (PO) if no feasible allocation makes some member better
off and no member worse off. We compare three welfare objectives, the
utilitarian, egalitarian (max-min), and Nash social welfare:
\begin{equation}
W_{\mathrm{UT}}(A)=\sum_{i=1}^{n}U_i(A_i),\qquad
W_{\mathrm{EG}}(A)=\min_{i}U_i(A_i),\qquad
W_{\mathrm{NSW}}(A)=\Big(\prod_{i=1}^{n}U_i(A_i)\Big)^{1/n}.
\label{eq:rules}
\end{equation}

\textbf{Matroid structure.} Encode an allocation as a set
$X\subseteq\Omega=N\times\cF$ of member--food pairs, with
$A_i=\{f:(i,f)\in X\}$. Feasibility requires that (i) each food appears in at
most one pair and (ii) each member appears in at most $K$ pairs. Each
condition defines a partition matroid on $\Omega$, so the feasible
allocations are exactly the common independent sets of two matroids.
Section~\ref{sec:allocation} uses this structure to bound the quality of
greedy allocation.

\subsection{Utility Model}
\label{sec:utility}
The utility of food $f_j$ for member $i$ combines three encoders.

\textbf{Taste encoder.} The user embedding $\mathbf{e}_{u_i}$ is a TF-IDF
weighted mean of sentence embeddings of the member's previous-day foods, so
a user unseen in training is represented by the same function of their
history. We concatenate it with the food embedding $\mathbf{e}_{f_j}$, the
sentence embedding of the food's description, and pass the result through a
small feedforward network with GELU activations and layer normalization,
giving a taste representation $\mathbf{t}_{ij}\in\mathbb{R}^{128}$ that
captures preference independently of pantry or dietary state.

\textbf{Retrieval-augmented constraint encoder.} We represent dietary
requirements by retrieving evidence for a health query $q_i$ from a
nutrition corpus $D$, ranked by cosine similarity:
\begin{equation}
R_i=\TopK_{r\in D}\ \frac{E(q_i)^{\top}E(r)}{\lVert E(q_i)\rVert\,\lVert E(r)\rVert}.
\label{eq:retrieval}
\end{equation}
The retrieved passages are projected into the latent space and passed
through self-attention, so that, for example, a note on potassium
restriction can interact with one on kidney function, and are then pooled by
learned attention into a constraint vector $\mathbf{c}_i\in\mathbb{R}^{128}$.
Grounding this vector in retrieved evidence rather than in parameters alone
allows the explanation module (Section~\ref{sec:explainability-method}) to
cite the specific evidence behind a recommendation.

\textbf{Pantry (availability) encoder.} NHANES records no household
inventory, so we use what a member ate the day before as a proxy for what is
at hand. With $H_i$ the foods in member $i$'s previous-day record,
\begin{equation}
p_{ij}=\mathbf{1}\big[f_j\in H_i\big],
\label{eq:pantry}
\end{equation}
which we project into the same latent space as $\mathbf{t}_{ij}$ and
$\mathbf{c}_i$, giving the model a per-pair availability signal to weigh
against taste and dietary fit.

\textbf{Fusion and utility.} We stack the three vectors and apply multi-head
attention, so the model can learn how to weight each signal for a given
pair. The result is flattened, combined with explicit recency and popularity
features $\mathbf{z}_{ij}$, and passed through an MLP that outputs a neural
score $s_{ij}$. The utility combines this score with the raw pantry match and
popularity and subtracts a soft constraint penalty:
\begin{equation}
U_i(f_j)=\sigma\big(\underbrace{\alpha s_{ij}+\beta p_{ij}+\delta\Pop(f_j)}_{\ell_i(f_j)}-\lambda\,\SCP(u_i,f_j)\big),
\label{eq:utility}
\end{equation}
where $\alpha,\beta,\delta,\lambda$ weight the neural, pantry, popularity,
and penalty terms and $\ell_i(f_j)$ denotes the unpenalized logit.

\textbf{Soft constraint penalty.} The penalty is built from the same rules
$\cC_i$ that the safety filter enforces:
\begin{equation}
\SCP(u_i,f)=\sum_{c\in\cC_i}w_c\,\rho_\kappa\big(g_c(f)\big),
\qquad
\rho_\kappa(z)=\frac{1}{\kappa}\log\big(1+e^{\kappa z}\big),
\label{eq:scp}
\end{equation}
with rule weights $w_c>0$ and a temperature $\kappa>0$. Because
$\max(0,z)\leq\rho_\kappa(z)\leq\max(0,z)+\kappa^{-1}\ln 2$, the penalty is a
smoothed form of the exact penalty $\sum_c w_c\max(0,g_c(f))$ and converges
to it as $\kappa\to\infty$; the hard filter is the all-or-nothing form of the
same rules. With unit rule weights, $w_c=1$, and logistic noise in the
nutrient data, $-\kappa\,\SCP$ is the log-probability that the food is safe, so the penalized utility approximates
the probability that a food is both chosen and safe when $\lambda=\kappa$
(Appendix~\ref{app:scp}); this reading is a special case and is not needed
for Proposition~\ref{prop:scp}. The penalty supplies gradients that a
step-function rule cannot, but it cannot replace the filter.

\begin{proposition}[Limits of a soft penalty]
\label{prop:scp}
Fix member $i$. Let $|\ell_i(f)|\leq L$ for all $f$, and let
$w_{\min}=\min_{c\in\cC_i}w_c$.
\begin{props}
\item \emph{Clear violations.} Let $\gamma>0$. If
\[
\lambda\Big(w_{\min}\,\rho_\kappa(\gamma)-\frac{\ln 2}{\kappa}\sum_{c\in\cC_i}w_c\Big)>2L,
\]
then every food that violates a rule of member $i$ by at least $\gamma$
ranks below every safe food.
\item \emph{Near-threshold violations.} Let $f^+$ be safe with $g_c(f^+)=0$,
and let $f^-$ differ from $f^+$ only in rule $c$, with $g_c(f^-)=\eta>0$.
For every finite $\lambda$, the unsafe $f^-$ ranks above $f^+$ whenever
$\ell_i(f^-)-\ell_i(f^+)>\lambda w_c\eta$.
\end{props}
\end{proposition}

Part~(a) shows that a learned penalty can push clear violations down the
ranking, but only when the logit is bounded, and its condition fails as
$\gamma\to0$. Part~(b) shows that, for any finite weight, foods just over a
limit cannot be separated from foods just under it, since the required
margin $\lambda w_c\eta$ vanishes as $\eta\to0$. Certifying
near-threshold foods therefore requires the deterministic filter of
Section~\ref{sec:safety}. Proofs are given in Appendix~\ref{app:proofs}.

\subsection{Allocation Rules and Greedy Allocation}
\label{sec:allocation}
Independent top-$K$ lists ignore that the pantry is shared, so FairFRESH
selects a feasible allocation with one of the welfare rules of
Eq.~\ref{eq:rules}. The utilitarian rule maximizes total utility but may leave a member
underserved; the egalitarian (max-min) rule protects the worst-off member at
some cost in total utility; and the Nash rule balances the two, because the
product falls sharply when any member's utility approaches zero.
Proposition~\ref{prop:bargain} below characterizes this balance.

\textbf{Greedy allocation.} Exact optimization is intractable for the Nash
and egalitarian rules (Section~\ref{sec:related}), so all three rules share
one greedy procedure (Algorithm~\ref{alg:greedy}). Starting from empty
bundles, each step gives one member its best remaining food, until every
member holds $K$ items or no food remains; the rules differ only in which
member moves next. The utilitarian rule picks the member whose best
remaining food has the highest utility, which adds the feasible pair of
largest marginal gain in $W_{\mathrm{UT}}$. The Nash rule does the same for
the smoothed objective
\begin{equation}
\Phi_\epsilon(A)=\sum_{i=1}^{n}\log\Big(1+\frac{U_i(A_i)}{\epsilon}\Big),\qquad \epsilon>0,
\label{eq:smoothnash}
\end{equation}
which keeps gains finite while bundles are empty; because the set of
allocations is finite, for sufficiently small $\epsilon$ its maximizers also
maximize the Nash product. The gain of giving $f$ to $i$ is
$\log\big(1+U_i(f)/(\epsilon+U_i(A_i))\big)$, so the Nash rule picks the
member with the largest ratio $U_i(f_i^\star)/(\epsilon+U_i(A_i))$, where
$f_i^\star$ is that member's best remaining food.

The egalitarian rule gives
the next food to the currently worst-off member, a max-min heuristic rather
than a marginal-gain step. Greedy allocation is thus an adaptive picking
sequence, of which round robin is the fixed-order counterpart. With each
member's candidates sorted once, Algorithm~\ref{alg:greedy} runs in
$O(nm\log m+n^{2}K)$ time.

\begin{algorithm}[t]
\caption{Greedy household allocation as an adaptive picking sequence}
\label{alg:greedy}
\begin{algorithmic}[1]
\Require utilities $U_i(f)$; cap $K$; rule $\in\{\mathrm{UT},\mathrm{EG},\mathrm{NSW}\}$; $\epsilon>0$
\State $A_i\gets\varnothing$ for every member $i$; \ $F_{\mathrm{rem}}\gets\cF$
\While{$F_{\mathrm{rem}}\neq\varnothing$ and some member has $|A_i|<K$}
  \State $I\gets\{i:|A_i|<K\}$; \ $f_i^\star\gets\arg\max_{f\in F_{\mathrm{rem}}}U_i(f)$ for every $i\in I$
  \State $i^\star\gets\begin{cases}\arg\max_{i\in I}U_i(f_i^\star) & \text{if UT}\\ \arg\min_{i\in I}U_i(A_i) & \text{if EG}\\ \arg\max_{i\in I}U_i(f_i^\star)/(\epsilon+U_i(A_i)) & \text{if NSW}\end{cases}$ \Comment{ties: lowest index}
  \State $A_{i^\star}\gets A_{i^\star}\cup\{f^\star_{i^\star}\}$; \ $F_{\mathrm{rem}}\gets F_{\mathrm{rem}}\setminus\{f^\star_{i^\star}\}$
\EndWhile
\State $\hat A_i\gets A_i\setminus V_i$ for every $i$ \Comment{safety filter, Section~\ref{sec:safety}}
\State \Return $\hat A$
\end{algorithmic}
\end{algorithm}

\begin{proposition}[Greedy allocation]
\label{prop:greedy}
Consider Algorithm~\ref{alg:greedy} on the feasible allocations of
Section~\ref{sec:problem}, before filtering and without a stock limit.
\begin{props}
\item \emph{Nash.} $\Phi_\epsilon$ is monotone and submodular, and greedy
Nash allocation attains at least $1/3$ of its
maximum~\cite{fisher1978analysis}.
\item \emph{Utilitarian.} Greedy utilitarian allocation attains at least
$1/2$ of the optimum~\cite{korte1978analysis}. The optimum is a
maximum-weight bipartite $b$-matching, computable in polynomial time.
\item \emph{Egalitarian.} Approximating the egalitarian optimum within any
factor below $2$ is NP-hard~\cite{bezakova2005allocating}. We claim no
guarantee for the greedy egalitarian rule.
\end{props}
\end{proposition}

Part~(a) bounds the smoothed objective, not the Nash product. Neither bound
carries over to the filtered output in general, since removing unsafe foods
can lower utilities unevenly; in our experiments, however, the filter
removes nothing (Section~\ref{sec:safety}). The experiments also enforce a
pantry stock (Section~\ref{sec:experiments}), a knapsack constraint outside
the matroid model, so Appendix~\ref{app:results} compares greedy allocation
with exact optima instead.

The EF1 guarantee of the exact Nash
optimum~\cite{caragiannis2019unreasonable} does not transfer either: it
holds for complete allocations without constraints, whereas under our cap
the Nash optimum is PO within the feasible set but not known to be EF1.
Round robin over safe foods, stopped after $K$ rounds, is EF1 under
safe-part valuations and is the natural benchmark for an EF1 audit
(Appendix~\ref{app:results}).

\textbf{A bargaining view of the three rules.} Household allocation is an
$n$-person bargaining problem $(F,d)$~\cite[Ch.~26]{narahari2014game}: $F$ is
the set of utility vectors $\big(U_1(A_1),\ldots,U_n(A_n)\big)$ of feasible
allocations, and the disagreement point $d=\mathbf{0}$ is what every member
receives if no allocation is agreed, an empty bundle. The utilitarian and
egalitarian rules correspond to the utilitarian and egalitarian solutions of
bargaining theory~\cite[Sec.~26.6]{narahari2014game}, and the Nash rule to the
Nash bargaining solution of $(F,\mathbf{0})$~\cite{nash1950bargaining}, the
only one of the three characterized by Nash's axioms. The smoothed
objective is the Nash product with disagreement point $-\epsilon\mathbf{1}$,
since $\Phi_\epsilon(A)=\sum_i\log\big(U_i(A_i)+\epsilon\big)-n\log\epsilon$.
Nash's axioms assume a convex feasible set. Convexity arises here
when the household rotates allocations over repeated meals: a rotation that
uses allocation $A$ on a fraction of the meals yields expected utilities in
the convex hull $\mathrm{conv}(F)$, which is the rationale for convexity in
bargaining theory~\cite[Sec.~26.2]{narahari2014game}.

\begin{proposition}[Bargaining properties]
\label{prop:bargain}
Let $F$ be the set of utility vectors of feasible allocations.
\begin{props}
\item \emph{Scale covariance.} For constants $a_i>0$, replacing each $U_i$ by
$a_iU_i$ leaves the Nash-optimal allocations unchanged. The utilitarian and
egalitarian rules lack this property, both exactly and greedily.
\item \emph{Synthesis.} Suppose $F$ contains a vector with positive entries,
and let $x^\star$ maximize $\prod_ix_i$ over $\mathrm{conv}(F)$. With
$\lambda_i=1/x^\star_i$,
\[
x^\star\in\arg\max_{x\in\mathrm{conv}(F)}\ \sum_i\lambda_ix_i
\qquad\text{and}\qquad
\lambda_1x^\star_1=\cdots=\lambda_nx^\star_n.
\]
Conversely, any positive $x^\star\in\mathrm{conv}(F)$ with both properties
for some $\lambda>0$ is the Nash solution.
\item \emph{Greedy step.} Let $\lambda_i=1/(\epsilon+U_i(A_i))$. The Nash
rule of Algorithm~\ref{alg:greedy} selects the member maximizing
$\lambda_iU_i(f_i^\star)$. As $\epsilon\to0$, this choice is invariant to
rescaling among members that already hold an item.
\end{props}
\end{proposition}

Part~(a) matters because utilities are model estimates whose scale may
differ across members, for instance when a longer history yields more
confident scores. The Nash optimum does not depend on how utilities compare
across members~\cite{kaneko1979nash}, whereas there are instances in which
the utilitarian rule favors a member whose utilities are scaled up and the
egalitarian rule one whose utilities are scaled down
(Appendix~\ref{app:proofs}).

Part~(b) extends Proposition~26.1 of~\cite{narahari2014game} from two
members to $n$: the Nash rule is a utilitarian rule that weights each member
by the inverse of its utility, and an egalitarian rule in those weighted
units, which is why it lands between the other two rules in
Table~\ref{tab:allocation}. Part~(c) shows that greedy Nash allocation
applies this reweighting one pick at a time, with weights
$\lambda_i=1/(x_i+\epsilon)$ for disagreement point $-\epsilon\mathbf{1}$;
rescaling can affect only the order in which members with empty bundles
make their first pick. It is scale-free once every
member holds an item, and fixing the order of first picks removes the
remaining dependence as $\epsilon\to0$. Section~\ref{sec:results_allocation}
tests~(a) and~(c).

\subsection{Symbolic Safety Filter}
\label{sec:safety}
The penalty in Eq.~\ref{eq:utility} shapes the ranking but, by
Proposition~\ref{prop:scp}, cannot guarantee safety. After allocation, a
deterministic verifier therefore removes from each bundle every food that
breaks a rule of its member, regardless of score:
\begin{equation}
\hat{A}_i=A_i\setminus V_i=A_i\cap S_i.
\label{eq:filter}
\end{equation}
In the allocation experiments, the rules limit sodium to 500\,mg per
serving for members with hypertension and carbohydrates to 50\,g per
serving for members with diabetes; the coverage evaluation of
Section~\ref{sec:results_rag} adds chronic kidney disease (CKD) and
cardiovascular disease (CVD). The filter guarantees compliance with these
encoded per-serving rules; it does not establish clinical safety, which
would also depend on rule completeness, nutrient data, portion sizes, and
whole-day intake.

We measure compliance with the constraint satisfaction rate of an
allocation $B$,
\begin{equation}
\CSR(B)=\frac{1}{n}\sum_{i=1}^{n}\frac{|B_i\setminus V_i|}{|B_i|}.
\label{eq:csr}
\end{equation}
Before filtering, $\CSR(A)$ measures what the learned components achieve
alone; after filtering, $\CSR(\hat A)=1$ whenever the filter has the full
rule set. In the ablation of Section~\ref{sec:results_ablation}, rules are
removed from the filter, and CSR, still measured against the full rule set,
can fall below one.

\textbf{Filter placement.} The filter runs after allocation so that it stays
independent of the allocation rule and can audit any upstream policy. The
cost is that removed foods are not replaced, so a verified bundle can hold
fewer than $K$ items, and the loss falls mostly on members with the most
restrictive rules. The alternative is to allocate within the safe sets
$S_i$, which never does worse at the optimum:

\begin{proposition}[Safe-set allocation at the optimum]
\label{prop:placement}
For any welfare objective $W$, let $A^{\mathrm{all}}$ maximize $W$ over all
feasible allocations, and let $A^{\mathrm{safe}}$ maximize $W$ over feasible
allocations with $A_i\subseteq S_i$ for every $i$. Then
\[
W(A^{\mathrm{safe}})\geq W(\hat{A}^{\mathrm{all}}),
\]
where $\hat{A}^{\mathrm{all}}$ is $A^{\mathrm{all}}$ after filtering.
\end{proposition}

In practice the filter has nothing to remove: across the 219,539
recommendations of the held-out allocation runs, the soft penalty already
keeps every recommendation within the rules, so $\CSR(A)=\CSR(\hat A)=1$.
The shared-meal extension of Section~\ref{sec:shared_meals} allocates within
the safe sets.

\subsection{Post-Hoc Explainability Module}
\label{sec:explainability-method}
The verified recommendations $\hat{A}_i$ feed a post-hoc explainability
module, which computes attributions only after the utility, allocation, and
safety decisions are final. FairFRESH explains why a food reaches a member by
treating the pipeline as a cooperative game whose players $\cK$ are the five
components that shape a member's utility and allocation: Taste, Pantry, RAG,
Popularity, and Fairness. Safety is not a player: the filter is a hard
constraint that stays active in every coalition, since crediting it with
utility would treat safety as something that can be traded off.

The value function is defined for each member--food pair that the full
pipeline recommends. For a coalition $T\subseteq\cK$, we recompute the
utilities and rerun the allocation with only the components in $T$ active;
with $U^T$ and $A^T$ the resulting utilities and allocation, the value of
the pair $(i,f)$ is $v_{if}(T)=U^T_i(f)$ if $f\in A^T_i$ and $0$ otherwise.

Removing Taste sets the user embedding to zero, removing Pantry sets
$p_{ij}$ to zero, removing RAG replaces the retrieved evidence $R_i$ of
Eq.~\ref{eq:retrieval} by padding, removing Popularity sets $\Pop(f_j)$ to
zero, and removing Fairness replaces exclusive Nash allocation by
independent top-$K$ lists. The dietary rules stay active in every
coalition, and unsafe foods are excluded before allocation. Thus
$v(T\cup\{k\})-v(T)$ is the utility that component $k$ adds once the
components in $T$ are present, and attributions are averaged over pairs.
Averaging this marginal over all orderings gives the Shapley attribution of
component $k$,
\begin{equation}
\phi_k=\sum_{T\subseteq\cK\setminus\{k\}}\frac{|T|!\,(|\cK|-|T|-1)!}{|\cK|!}\big[v(T\cup\{k\})-v(T)\big],
\label{eq:shapley}
\end{equation}
and averaging it uniformly over coalitions gives the Banzhaf index,
\begin{equation}
\psi_k=\frac{1}{2^{|\cK|-1}}\sum_{T\subseteq\cK\setminus\{k\}}\big[v(T\cup\{k\})-v(T)\big].
\label{eq:banzhaf}
\end{equation}
Both indices are linear in $v$~\cite[Ch.~29]{narahari2014game}, so the
mean attributions over households that Section~\ref{sec:explainability}
reports are the attributions of the household-averaged game. The Shapley
value is efficient, $\sum_{k}\phi_k=v(\cK)-v(\varnothing)$; the Banzhaf index
is not.

The two indices also weight coalitions differently. The Shapley value gives
each coalition size $|T|\in\{0,\ldots,4\}$ total weight $1/5$, and the
Banzhaf index gives size $s$ weight $\binom{4}{s}/16$, so the coalitions with
none or all of the other components count for $1/5$ under Shapley but
$1/16$ under Banzhaf. Since every size receives Banzhaf weight at least
$1/16$, a component whose mean marginal contribution is non-negative at every
coalition size satisfies $\psi_k\geq\frac{5}{16}\phi_k$.

We pass these attributions, together with the retrieved evidence, the pantry
overlap, and allocation metadata, to a lightweight language model whose only
task is to render them in natural language. The generated text is checked
against the same evidence and rejected if it introduces a condition,
ingredient, nutrient, or line of reasoning that was not used. The module can
therefore explain the pipeline's decisions but cannot change ranking,
allocation, or safety.

\section{Experimental Setup}
\label{sec:experiments}

We evaluate FairFRESH on households of two to six members built from a
national health and diet survey. We first measure the individual accuracy
of the utility model and then compare allocation strategies that share this
model, so that differences between strategies come from the allocation
layer alone.

\subsection{Data}
User profiles come from the National Health and Nutrition Examination
Survey (NHANES), pooling public releases from cycle C through the March 2023
release (2003--2023). For each respondent, we join the demographic,
medical-condition, body-measurement, blood-pressure, and two-day
dietary-recall components by sequence number; BMI, diabetes status, and
hypertension are read directly from these records. Each recall lists the
USDA food codes the respondent consumed; together they form the 8,704-item
catalog $\cF$. The first recall day is a respondent's history and the second
the test interactions. The utility model is trained on 1,000 respondents,
with each user's first-day foods as input and second-day foods as positives.
The allocation experiments use 1,000 other respondents, drawn from the
60,814 eligible respondents not used in training, and the individual
comparison of Section~\ref{sec:results_baselines} uses 400 further test
users.

Because NHANES records individuals rather than households, we form
households of $n=2$--$6$ by sampling respondents without replacement within
each size (100 households per size); households of different sizes may
share respondents. The utility model is trained once and held fixed across
all allocation strategies.

\textbf{Availability and stock.} NHANES records no household inventory, so
we use two proxies. In the utility model, a food is available to a member if
the member reported it on the first recall day (Eq.~\ref{eq:pantry}). In
allocation, each household receives a stock of 2 to 7 units of each of the
100 most frequent words in food descriptions, which stand in for
ingredients; a food requires one unit of each of its description words, and
strategies B--E allocate a food only if the remaining stock can supply it,
deducting its requirements.

\subsection{Baselines}
For individual accuracy, we compare
FairFRESH's utility model with BPRMF~\cite{Rendle2009},
NeuMF~\cite{He2017NeuMF}, LightGCN~\cite{He2020LightGCN},
GRU4Rec~\cite{Hidasi2016GRU4Rec}, and SASRec~\cite{Kang2018SASRec}, and
with two heuristics: Popularity, and Repeat-D1, which ranks a user's
first-day foods by popularity ahead of all other foods. The ID-based models
are trained with the test users' first-day interactions, so every test user
has a trained embedding. For household
allocation, five strategies share the same utility model: (A) independent
top-$K$ lists per member, which may give the same food to several members;
(B) conflict resolution, which merges the members' top-$K$ lists in order of
utility and skips any food already held by another member; and the
(C) utilitarian, (D) egalitarian (max-min), and (E) Nash rules of
Section~\ref{sec:allocation}. Strategies B--E allocate each food to at most
one member, as Section~\ref{sec:problem} requires; strategy A does not and
serves only as an uncoordinated reference.

\subsection{Evaluation Protocol and Metrics}
We report accuracy against held-out interactions
(Recall@20 and NDCG@20); the pantry utilization rate (PUR), the share of a
member's recommendations drawn from the household's first-day foods,
normalized by $\min(|P|,K)$;
equity over members' utilities, namely the Gini coefficient, minimum
utility, and Nash social welfare (NSW, the geometric mean of member
utilities); and dietary coverage, namely retrieval success, evidence
relevance, and CSR (Eq.~\ref{eq:csr}). Each member receives $K=20$ items,
matching the metric cutoff, and the Nash rule uses $\epsilon=10^{-5}$.

Minimum utility and NSW are the objectives of the egalitarian and Nash
rules, so we compare rules mainly on the other metrics. Equity is computed
on predicted utilities, which assumes they are comparable across members;
Section~\ref{sec:results_allocation} tests how much each rule depends on
this assumption. We report 95\% bootstrap confidence intervals (2,000
resamples over households, or over users for Table~\ref{tab:baselines}) and
compare strategies with paired $t$-tests over households, reporting the
paired effect size $d_z$. Training and baseline settings are in
Appendix~\ref{app:results}.

Tables~\ref{tab:allocation} and~\ref{tab:ablation} use the same 100
four-person households, so the full model in the ablation is Nash
allocation in Table~\ref{tab:allocation}. The size sweep draws its own
households; at $n=4$, its Nash results (Recall@20 0.218, Gini 0.106) differ
from Table~\ref{tab:allocation} (0.197 and 0.112) within sampling variation,
with overlapping confidence intervals. The attribution analysis uses 500
households of two to six members (2,000 members) from the same held-out
pool.

\section{Results}
\label{sec:results}

\subsection{RQ-A: Does Availability Provide Useful Signal?}
\label{sec:results_pantry}
On the 400 test users, zeroing the availability input of the trained
utility model, a state it sees for every food a member did not eat the day
before, cuts Recall@20 from 0.293 to 0.083 and NDCG@20 from 0.311 to 0.091.
Replacing each member's first-day foods with a random set of the same size
gives 0.045 and 0.027. With the real input, 96.6\% of the top 20
recommendations are first-day foods, against 8.5\% when the input is
zeroed. Most of the model's accuracy therefore comes from predicting that
people eat again the foods they ate the day before.

\subsection{Comparison with Recommender Baselines}
\label{sec:results_baselines}
Table~\ref{tab:baselines} reports individual-level accuracy. Repeat-D1 is
the most accurate method (Recall@20 0.321), which confirms how much of the
signal lies in previous-day consumption. The utility model is second
(0.293) and above every learned baseline; the strongest of these are
LightGCN (0.244) and BPRMF (0.220), and Popularity alone reaches 0.235.
Individual ranking is therefore not where FairFRESH's contribution lies. The
allocation layer needs a utility for every member--food pair that combines
availability with taste, popularity, and the dietary penalty, and the
utility model supplies these at an accuracy close to the repeat heuristic.

To test what the learned utilities add, we allocate the held-out households
of Table~\ref{tab:allocation} with Repeat-D1 scores as utilities. Nash
allocation then reaches a slightly higher Recall@20 (0.210 against 0.197)
but leaves dietary compliance entirely to the filter: 2.7\% of allocated
foods break a member's rule, and after filtering, members with diabetes or
hypertension hold 17.8 items instead of 20 ($d_z=0.87$, $p<10^{-13}$, paired
over households). With learned utilities, the dietary penalty keeps every
allocated food within the rules before filtering, at a small cost in
accuracy.

\begin{table}[htbp]
\caption{Individual-level accuracy on 400 test users (mean over five seeds
for trained baselines; 95\% user-bootstrap confidence intervals in
brackets).}
\label{tab:baselines}
\centering
\small
\begin{tabular}{lcc}
\toprule
\textbf{Method} & \textbf{Recall@20} & \textbf{NDCG@20} \\
\midrule
Popularity & 0.235 [0.221, 0.248] & 0.253 [0.241, 0.265] \\
Repeat-D1 & \textbf{0.321} [0.304, 0.339] & \textbf{0.376} [0.358, 0.393] \\
\midrule
BPRMF & 0.220 [0.206, 0.234] & 0.234 [0.220, 0.249] \\
NeuMF & 0.111 [0.105, 0.117] & 0.141 [0.133, 0.149] \\
LightGCN & 0.244 [0.231, 0.258] & 0.262 [0.250, 0.274] \\
GRU4Rec & 0.207 [0.195, 0.220] & 0.234 [0.220, 0.247] \\
SASRec & 0.012 [0.004, 0.022] & 0.005 [0.002, 0.008] \\
\midrule
FairFRESH utility model & 0.293 [0.276, 0.311] & 0.311 [0.293, 0.331] \\
\quad without availability input & 0.083 [0.074, 0.093] & 0.091 [0.080, 0.102] \\
\bottomrule
\end{tabular}
\end{table}

\subsection{RQ-B: Does Fairness Improve Household Recommendation?}
\label{sec:results_allocation}
Table~\ref{tab:allocation} compares the five strategies on 100 held-out
four-person households. Conflict resolution (B) is the least equitable
coordinated strategy on every equity metric. A member displaced from a
contested food moves down their own list, and because the lists hold only
each member's top 20 foods, displaced members can end with fewer than 20
items (Table~\ref{tab:scaling}, Appendix~\ref{app:results}). Among the three
allocation rules, the egalitarian rule is the most equal, the utilitarian
rule the least equal, and Nash allocation sits between them with the
highest Recall@20 of any coordinated strategy.

\begin{table}[htbp]
\caption{Household allocation on 100 held-out four-person households, with
95\% household-bootstrap confidence intervals. Higher is better except for
Gini. Bold marks the best coordinated strategy (B--E) in each column; A is
an uncoordinated reference that may give the same food to several
members.}
\label{tab:allocation}
\centering
\small
\setlength{\tabcolsep}{4pt}
\begin{tabular}{lcccc}
\toprule
\textbf{Strategy} & \textbf{Recall@20} & \textbf{Min Util.} & \textbf{Gini} & \textbf{NSW} \\
\midrule
A) Independent & 0.266 [0.250, 0.282] & 9.79 [9.36, 10.21] & 0.109 [0.100, 0.118] & 12.76 [12.42, 13.09] \\
B) Conflict resolution & 0.180 [0.165, 0.194] & 5.35 [5.00, 5.70] & 0.192 [0.180, 0.205] & 8.54 [8.25, 8.81] \\
C) Utilitarian & 0.194 [0.180, 0.209] & 7.64 [7.31, 7.96] & 0.146 [0.136, 0.155] & 10.72 [10.45, 10.99] \\
D) Max-Min & 0.188 [0.173, 0.202] & \textbf{8.75} [8.41, 9.08] & \textbf{0.088} [0.079, 0.096] & 10.56 [10.27, 10.87] \\
E) Nash & \textbf{0.197} [0.182, 0.211] & 8.56 [8.23, 8.89] & 0.112 [0.103, 0.121] & \textbf{10.98} [10.70, 11.27] \\
\bottomrule
\end{tabular}
\end{table}

Strategy A is more accurate (Recall@20 0.266) only because it gives
contested foods to several members at once, so its accuracy is a ceiling
rather than a feasible baseline. Paired over the 100 households, Nash
allocation lowers Gini relative to conflict resolution by 42\%, from 0.192
to 0.112 ($d_z=-1.94$, $p<10^{-34}$), raises NSW from 8.54 to 10.98
($d_z=3.57$, $p<10^{-57}$), and raises Recall@20 from 0.180 to 0.197
($d_z=0.44$, $p=3\times10^{-5}$). Strategies C, D, and E are EF1 in every
household, and conflict resolution in 92--99\% depending on household size.
Round robin over safe foods, the EF1 benchmark of
Section~\ref{sec:allocation}, is a strong simple alternative: on a separate
set of four-person households it matches Nash allocation on Recall@20
(0.216 against 0.218) with lower Gini (0.097 against 0.106) and lower NSW
(10.47 against 10.63; Appendix~\ref{app:results}).

\textbf{Robustness to utility scale.} Proposition~\ref{prop:bargain}(a)
predicts that the utilitarian and egalitarian rules respond to how one
member's utilities are scaled, while the Nash rule does not. We multiply
the utilities of one randomly chosen member of each household by
$a\in\{0.5,2\}$ and rerun allocation. Doubling a member's utilities changes
35\% of utilitarian assignments and raises that member's true utility by
16\%; halving them changes 41\% of egalitarian assignments and raises it by
13\%. Greedy Nash allocation changes 1.5\% of assignments at $a=0.5$ and
0.2\% at $a=2$, and the rescaled member gains nothing (changes of $-0.010$
and $+0.002$ in utility). With the order of first picks fixed
(Proposition~\ref{prop:bargain}(c)), no assignment changes. Round robin uses
only each member's ranking and is invariant to rescaling by construction.

\subsection{Scaling with Household Size}
\label{sec:results_scaling}

Fig.~\ref{fig:scaling} plots Gini and minimum utility against household
size for $n=2$--$6$; Table~\ref{tab:scaling} in Appendix~\ref{app:results}
reports all metrics.

\begin{figure}[t]
\centering
\includegraphics[width=0.85\textwidth]{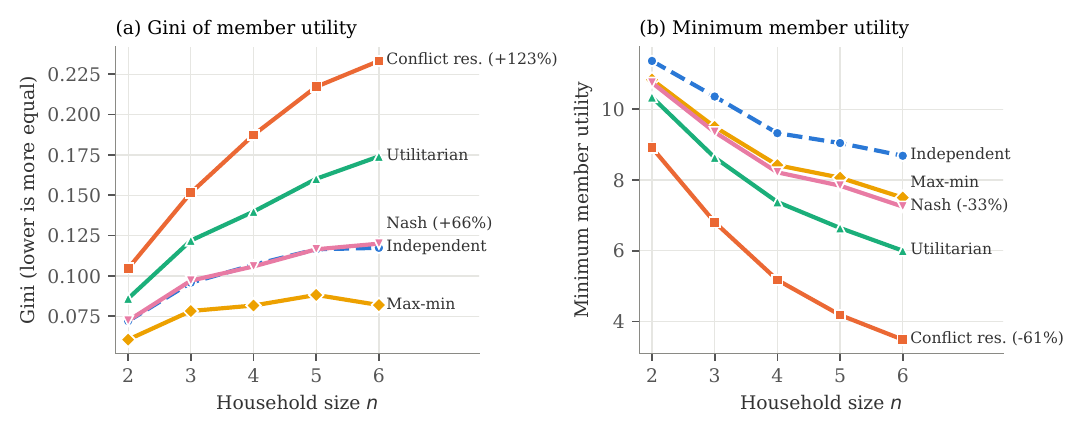}
\caption{Scaling with household size, plotted from Table~\ref{tab:scaling}.
(a) Gini of per-member utility (lower is more equal). (b) Minimum member
utility (higher is better). Percentages mark the changes from $n=2$ to
$n=6$ discussed in the text. Independent recommendation (dashed) is an
uncoordinated reference (Section~\ref{sec:results_allocation}).}
\label{fig:scaling}
\end{figure}

Conflict resolution degrades steadily as households grow
(Fig.~\ref{fig:scaling}): its minimum utility falls by 61\% from $n=2$ to
$n=6$ and its Gini more than doubles. Part of the loss is that displaced
members run out of candidates: the merged lists hold only each member's top
20 foods, so members receive 15.1 items on average at $n=2$ and 9.7 at
$n=6$. Nash allocation degrades more gently, its minimum utility falling by
33\%, so its lead over conflict resolution on that measure widens from 1.84
at $n=2$ to 3.76 at $n=6$. The egalitarian rule shows the smallest Gini
increase but ends with the lowest Recall@20 of any coordinated strategy
(0.160 at $n=6$). Nash allocation attains the highest NSW at every size, and
its Recall@20 is within 0.002 of the most accurate coordinated strategy
throughout. CSR remains 1.00 for every strategy and household size. The
advantage of fair allocation therefore grows with household size.

\subsection{RQ-C: Retrieval-Augmented Constraint Coverage}
\label{sec:results_rag}
Retrieval for the constraint encoder succeeds for all four conditions
tested. Before filtering, the retrieval-augmented encoder alone attains
constraint satisfaction of 100\% for diabetes and hypertension, 96\% for CKD,
and 98\% for cardiovascular disease (CVD), with evidence-relevance scores of
0.95, 0.94, 0.88, and 0.91. These figures measure retrieval and evidence
quality, not clinical validity.

\subsection{Component Ablation}
\label{sec:results_ablation}
Table~\ref{tab:ablation} compares the full model with four single-component
ablations under the protocol of Table~\ref{tab:allocation}; removing the
fairness optimizer means allocating by the utilitarian rule.

\begin{table}[htbp]
\caption{Component ablation on the held-out four-person households of
Table~\ref{tab:allocation}, under Nash allocation unless stated. The full
model's Recall@20 (0.197) lies below its individual-level value in
Table~\ref{tab:baselines} (0.293) because each food goes to one member. CSR
is measured against the full rule set.}
\label{tab:ablation}
\centering
\begin{tabular}{lcccc}
\toprule
\textbf{Model} & \textbf{Recall@20} & \textbf{CSR} & \textbf{Gini} & \textbf{NSW} \\
\midrule
Full FairFRESH & 0.197 & 1.000 & 0.112 & 10.98 \\
w/o Pantry (availability) Encoder & 0.064 & 1.000 & 0.160 & 5.58 \\
w/o Medical Constraint & 0.194 & 0.856 & 0.124 & 9.23 \\
w/o Fairness Optimizer (utilitarian) & 0.194 & 1.000 & 0.146 & 10.72 \\
w/o RAG Constraint Encoder & 0.200 & 1.000 & 0.087 & 10.06 \\
\bottomrule
\end{tabular}
\end{table}

Removing the availability encoder cuts Recall@20 from 0.197 to 0.064 and NSW
from 10.98 to 5.58, the largest loss of any component. Without the
medical-constraint input, which removes the dietary rules from both the
utility model and the filter, Recall@20 barely moves but 14.4\% of
recommendations break a member's dietary limit. Utilitarian allocation in
place of the fairness optimizer lowers Recall@20 slightly (0.194) and raises
Gini from 0.112 to 0.146. (In the attribution game of
Section~\ref{sec:explainability-method}, removing Fairness instead means
independent lists, which isolates its effect on utility.) Removing the
retrieval-augmented encoder leaves CSR at 1.000, since the encoder supplies
no rules to the filter; Recall@20 changes little (0.200), Gini falls to
0.087, and NSW falls to 10.06. With only diabetes and hypertension present,
its role is the coverage of Section~\ref{sec:results_rag} and the evidence
behind explanations.

\subsection{RQ-D: Are Post-Hoc Explanations Faithful?}
\label{sec:explainability}
Fig.~\ref{fig:explainability} reports attributions on 500 held-out
households of two to six members (2,000 members, 40,000 member--food
pairs). We measure faithfulness as the correlation between a component's
attribution and the utility lost when that component is removed, the
single-component case of the faithfulness correlation of Bhatt
et al.~\cite{bhatt2020evaluating}. Shapley attributions track these losses
closely ($r=0.763$) and Banzhaf attributions less closely ($r=0.597$;
Fig.~\ref{fig:explainability}a). The test removes components with the same
operator as the value function $v$, so it shows that the attributions agree
with the pipeline's own removal behavior, not causal faithfulness in a
stronger sense.

\begin{figure}[t]
\centering
\includegraphics[width=0.85\textwidth]{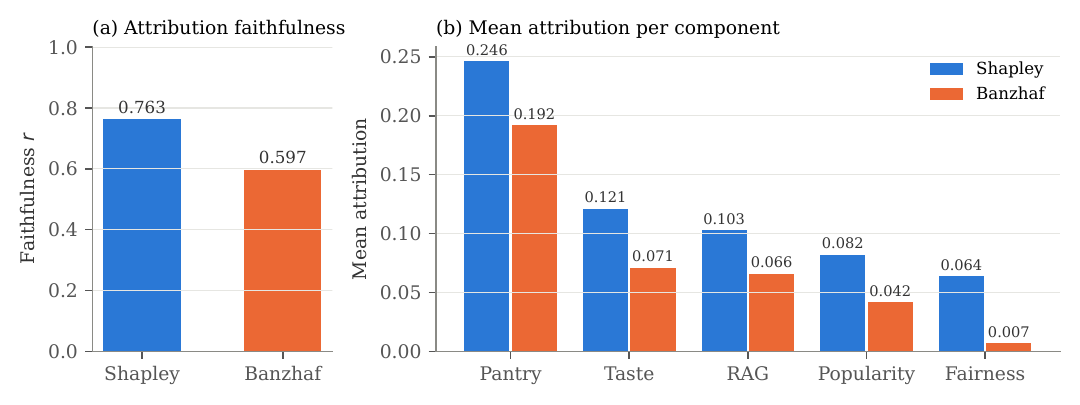}
\caption{Game-theoretic explainability on 500 held-out households of two to six members.
(a) Attribution faithfulness: correlation $r$ between each component's
attribution and the utility drop observed when that component is removed.
(b) Mean Shapley and Banzhaf attribution per component, sorted by Shapley
value. Safety is not a player, since the filter is a hard constraint
rather than a source of utility (Section~\ref{sec:explainability-method}).}
\label{fig:explainability}
\end{figure}

By efficiency, the mean Shapley values sum to
$v(\cK)-v(\varnothing)=0.616$; the Banzhaf indices sum to 0.378. Pantry
receives the most credit under both indices (Shapley 0.246, Banzhaf 0.192),
40\% of the Shapley total and about twice the credit of taste (0.121 and
0.071), in line with RQ-A and the ablation. RAG is third (0.103 and 0.066),
followed by Popularity (0.082 and 0.042) and Fairness (0.064 and 0.007).

Fairness redistributes utility rather than creating it, so its effect shows
in Gini and minimum utility (Tables~\ref{tab:allocation}
and~\ref{tab:scaling}) rather than in total utility. It is also the only
component that violates $\psi_k\geq\frac{5}{16}\phi_k$
(Section~\ref{sec:explainability-method}): its Shapley value of 0.064 would
require a Banzhaf index of at least 0.020 if its mean marginal contribution
were non-negative at every coalition size, against the observed 0.007.
Fairness therefore lowers total utility in coalitions of some intermediate
size and earns its Shapley credit in the smallest and largest coalitions,
which the Shapley value weights more heavily.

Four automated checks of 1,000 generated explanations against the
retrieved evidence give 98.7\% text faithfulness, 100.0\% evidence-citation
accuracy, 99.8\% completeness, and 99.1\% attribution consistency.

\section{Extension: Shared Household Meals}
\label{sec:shared_meals}

The formulation of Section~\ref{sec:problem} gives each member a disjoint
bundle, which models meals assembled from individual servings. Household
dining, however, often centers on dishes cooked once and eaten by everyone,
such as a common pot of rice or dal. A shared dish is a public good, while
the remaining slots stay private goods (Section~\ref{sec:related}), so we
choose shared dishes by a max-min criterion over members rather than by
envy, and evaluate the extension with simulated and with learned
utilities.

\paragraph{Setting.} As in the main experiments, the pantry is a stock
vector $q\geq0$ over ingredients and, for each dish $f$, a requirement vector
$r(f)$, which is one serving for an individual item and one household batch
for a shared dish. A dish is in stock if $r(f)\leq q$ componentwise, and
allocating it sets $q\gets q-r(f)$. We use two settings, each with 100
four-person households under five seeds, $K=6$, and $\epsilon=10^{-5}$:
3,597 Indian recipes from RecipeNLG~\cite{bien2020recipenlg} with simulated
utilities and randomly assigned exclusion rules, and the NHANES foods of the
held-out evaluation pool with the trained model's utilities and dietary rules.
Appendix~\ref{app:results} gives the details.

\paragraph{Two-stage allocation.} Algorithm~\ref{alg:shared_meal}
(Appendix~\ref{app:results}) first
chooses one grain $g^\star$ and one dal $d^\star$ that are cooked once and
added to every member's bundle, so their ingredients are charged to the
pantry a single time. Because every member eats a shared dish, it must lie
in every member's safe set, $g^\star,d^\star\in\bigcap_i S_i$, and the pair is
chosen to maximize the utility of the worst-off member. The remaining $K-2$
slots of each bundle are then filled by the Nash rule of
Algorithm~\ref{alg:greedy}, starting from the utility each member already
receives from the shared dishes and skipping any item that is unsafe for its
member or that the remaining pantry cannot supply.

If no safe grain--dal
pair is in stock, Stage~1 is skipped and the procedure reduces to Nash
allocation within the safe sets. As in Section~\ref{sec:problem}, each dish
is allocated to at most one member per meal, even when the stock could
supply a second serving; the stock limits which dishes can be prepared,
and the exclusivity rule keeps bundles distinct.

Stage~1 is exact: it enumerates all safe pairs in
$O(n|\mathcal{G}||\mathcal{L}|)$ time. Stage~2 is the Nash step of
Algorithm~\ref{alg:greedy} with $U_i(A_i)$ initialized to the utility of the
shared dishes, and it ends after at most $nK$ picks.
Proposition~\ref{prop:greedy} does not carry over: shared dishes break the
disjointness of Section~\ref{sec:problem}, and the stock constraint
$r(f)\leq q$ is a knapsack constraint rather than a matroid, so we claim no
approximation guarantee for Stage~2. Both stages choose only from safe
sets, so no post-hoc filter is needed, as Proposition~\ref{prop:placement}
recommends.

\paragraph{Rotating shared meals.} Stage~1 is a matrix
game~\cite[Ch.~9]{narahari2014game} whose rows are the safe pairs
$\pi=(g_\pi,d_\pi)\in\Pi$, whose columns are the members, and whose payoff is
$M_{\pi i}=U_i(g_\pi)+U_i(d_\pi)$: a planner chooses a pair and an adversary
chooses the member on whom it is judged. Stage~1 returns the planner's
maxmin value in pure strategies~\cite[Sec.~6.6]{narahari2014game}. A
household that plans shared meals for $T$ days need not cook one pair every
day, and in practice rarely does. A rotation that cooks pair $\pi$ on a
fraction $p_\pi$ of the days is a mixed strategy of the planner, under which
the worst-off member's average utility from the shared dishes is
$\min_i\sum_\pi p_\pi M_{\pi i}$.

\begin{proposition}[Rotating shared meals]
\label{prop:rotation}
Consider the rotation problem without stock limits. Let $v_{\mathrm{pure}}$
be the best worst-off utility from a single pair and $v^\star$ the best from
a rotation:
\[
v_{\mathrm{pure}}=\max_{\pi\in\Pi}\,\min_i M_{\pi i},
\qquad
v^\star=\max_{p\in\Delta(\Pi)}\,\min_i\sum_{\pi}p_\pi M_{\pi i}.
\]
\begin{props}
\item \emph{Minimax.} $v^\star\geq v_{\mathrm{pure}}$, and
\[
v^\star=\min_{w\in\Delta(N)}\,\max_{\pi\in\Pi}\,\sum_iw_iM_{\pi i}.
\]
An optimal rotation $p^\star$ and member weighting $w^\star$ solve a linear
program and its dual.
\item \emph{Support.} Some optimal rotation uses at most $n$ distinct pairs.
\item \emph{Binding members.} If $w^\star_i>0$, member $i$ receives exactly
$v^\star$ under $p^\star$. If $p^\star_\pi>0$, pair $\pi$ maximizes
$\sum_iw^\star_iM_{\pi i}$.
\end{props}
\end{proposition}

Part~(a) follows from the minimax theorem~\cite[Thm.~9.2]{narahari2014game}:
rotation can only help the worst-off member, and the best rotation is
computed exactly. By~(b), a four-person household needs at most four shared
meals. By~(c), $w^\star$ names the members whose preferences bind the plan,
which the explanation module can report. Every pair in $\Pi$ is safe for
every member, so every rotation is safe.

Without stock limits we round
$p^\star$ to a $T$-day plan by largest remainders. Stock limits over the plan,
$T\sum_\pi p_\pi\big(r(g_\pi)+r(d_\pi)\big)\leq q$, are linear, so the problem
remains a linear program, and each binding stock constraint adds at most one
pair to the bound in~(b). With stock limits, the linear program still gives
the best fractional rotation, but rounding can violate the stock. We
therefore compute the $T$-day plan exactly as an integer program and
compare it with the best pair that can be cooked on all $T$ days.

\paragraph{Results.} Table~\ref{tab:shared_meals} compares three strategies
in each setting over 500 households. Per-item utilities are min-max
normalized to $[0,1]$, so a bundle's utility lies between 0 and 6.
\emph{Unstructured} allocation applies the single-serving procedure of
Section~\ref{sec:allocation} with no meal structure. \emph{Structured
individual} allocation gives each member, in index order, the safe and
in-stock grain--dal pair with the highest utility for that member, cooked
for that member alone, and fills the remaining slots by the Nash rule.
\emph{Structured shared} allocation is Algorithm~\ref{alg:shared_meal}.
Full meal is the share of members whose bundle contains both a grain and a
legume, and pantry used is the share of the pantry stock consumed.

\begin{table}[htbp]
\caption{Shared versus individual meals on four-person households ($K=6$;
100 households under each of five seeds). Values are means over seeds;
standard deviations across seeds are at most 0.05 for utilities, 0.003 for
Gini, 2.3 points for full meal, and 0.12 points for pantry used. Per-item
utilities are min-max normalized. Higher is better except for Gini and
pantry used. Bold marks the best value in each column within a setting.}
\label{tab:shared_meals}
\centering
\begin{tabular}{lccccc}
\toprule
\textbf{Strategy} & \textbf{Mean Util.} & \textbf{Min Util.} & \textbf{Gini} & \textbf{Full Meal} & \textbf{Pantry Used} \\
\midrule
\multicolumn{6}{l}{\emph{Indian recipes, simulated utilities}} \\
Unstructured & \textbf{5.34} & \textbf{4.92} & 0.029 & 28.6\% & 4.59\% \\
Structured individual & 5.26 & 4.86 & 0.030 & \textbf{100\%} & 5.12\% \\
Structured shared & 5.06 & 4.74 & \textbf{0.024} & \textbf{100\%} & \textbf{3.47\%} \\
\midrule
\multicolumn{6}{l}{\emph{NHANES foods, learned utilities}} \\
Unstructured & \textbf{5.58} & \textbf{5.10} & \textbf{0.029} & 0.6\% & 10.74\% \\
Structured individual & 4.45 & 3.97 & 0.045 & \textbf{100\%} & 11.63\% \\
Structured shared & 4.32 & 4.01 & 0.030 & \textbf{100\%} & \textbf{8.02\%} \\
\bottomrule
\end{tabular}
\end{table}

Unstructured allocation attains the highest mean and minimum utility but
rarely forms a complete meal: 28.6\% of members receive both a grain and a
legume with Indian recipes, and 0.6\% with NHANES foods, where grains and
legumes are a small part of the catalog. Individual servings complete every
meal but raise pantry use, since each member's grain and dal are cooked
separately. Shared meals also complete every meal and use the least stock:
32\% less than individual meals with Indian recipes and 31\% less with
NHANES foods. They also lower Gini, from 0.030 to 0.024 and from 0.045 to
0.030, while mean utility falls by 4\% and 3\%, since a shared dish must
suit every member. All three differences are significant in both settings
($p<10^{-15}$, paired over 500 households; $|d_z|$ between 0.56 and 4.06).
Without the all-member safety constraint of Stage~1, 11.8\% of members are
served a shared dish that violates one of their rules with simulated
utilities, against 0.0\% with learned utilities, whose soft penalty already
scores unsafe dishes low.

Appendix~\ref{app:results} (Table~\ref{tab:worked_example}) shows both
allocations for one held-out NHANES household. The shared white bread and
pinto beans are safe for every member, including one with hypertension,
every member gets a complete meal, and pantry use falls from 14.4\% to
10.7\%, at a cost of 0.6 to 1.6 in each member's utility.

\paragraph{Rotation.} Over a seven-day plan in the Indian-recipe setting
(68,523 to 464,610 safe pairs per household), the rounded optimal rotation
raises the worst-off member's average utility from shared dishes from 1.329
to 1.365, improves on the fixed pair in 88.6\% of households, and uses 2.6
distinct pairs on average; $w^\star$ is positive for 3.2 of the 4 members.

Depletion matters more: cooking the Stage~1 pair every day exceeds the stock
in 99.4\% of households. Under the stock constraint, we compute the
seven-day plan exactly as an integer program over the linear program's
support and the 200 pairs with the highest minimum utility; every instance
solves to optimality. The integer plan gives the worst-off member 1.334
against 0.978 for the best pair that can be cooked all week, a gain of 38\%,
improves on that pair in every household, and lies within 0.005 of the
linear-programming bound. It uses 5.0 distinct pairs on average, more than
$n=4$, as the stock constraints allow. With NHANES foods, only 150 of 500
households have any pair that can be cooked all week; in those, the
integer rotation improves on it by 43\%, and every household has a feasible
rotation. Under a finite pantry, rotation is usually the only feasible
weekly plan.

\section{Discussion}
\label{sec:discussion}
Because the utility model is fixed while the allocation rule varies, the
equity differences we report come from allocation alone, and the default
procedure, conflict resolution, is the least equitable rule we tested.

\textbf{Nash allocation and round robin.} Round robin performs almost as
well as Nash allocation in our experiments and is EF1 by construction. The
two rules use different information. Round robin uses only each member's
ranking, so it is scale-free but cannot tell a strong preference from a weak
one; the Nash rule uses utility intensities, its exact optimum is
Pareto-optimal and maximizes NSW, and greedy Nash allocation attains higher
NSW than round robin at the same accuracy (10.63 against 10.47). When
utility differences are meaningful, the Nash rule trades them off across
members and round robin cannot; when only rankings can be trusted, round
robin is the safer simple choice.

\textbf{Strategic reporting.} Stated preferences enter the utility model,
so a member could misreport them. Households do not use money, so VCG
transfers are unavailable~\cite[Ch.~18]{narahari2014game}, and by the
Gibbard--Satterthwaite theorem~\cite[Ch.~17]{narahari2014game} a
strategy-proof deterministic rule that can select every safe shared pair is
dictatorial when preferences over three or more pairs are unrestricted.
Pair utilities in Stage~1 are additive over a grain and a dal, which
restricts the preference domain, so this is a conditional impossibility
rather than a proof that Stage~1 can be manipulated. By
Proposition~\ref{prop:bargain}(a), rescaling one's reported utilities can
pay under the utilitarian rule (by exaggerating) and the egalitarian rule
(by understating), but not under the exact Nash rule; scale invariance does
not make the Nash rule strategy-proof against other misreports.

\textbf{Limitations.} Households are random groupings of NHANES
respondents, availability is measured by previous-day consumption, and the
allocation stock uses description words as ingredients, so households,
availability, and stock are all proxies. The utility model was trained
without a validation set, and the baselines were not tuned. In the
shared-meal extension, pantry quantities are approximated, the
Indian-recipe utilities are simulated, and Stage~2 has no approximation
guarantee. Equity is measured on predicted rather than realized utilities.
The greedy egalitarian rule has no guarantee. The main experiments filter after allocation; allocation within
safe sets, which Proposition~\ref{prop:placement} favors, is evaluated only
for shared meals. Dietary coverage is assessed through retrieval and
evidence relevance rather than clinical review, the allocation experiments
cover only diabetes and hypertension, and CKD and CVD coverage (96\% and
98\%) falls short of the 100\% reached for those two conditions. The
explanation metrics are automated checks, not human judgments.

\section{Conclusion and Future Work}
\label{sec:conclusion}
We introduced FairFRESH, a neuro-symbolic framework for household food
recommendation. Existing food recommenders score each member in isolation,
although a shared pantry turns independently high-scoring recommendations
into competition for the same foods. FairFRESH instead treats household
recommendation as constrained fair division. A learned utility model scores
every member--food pair, an explicit welfare rule allocates foods from the
resulting utility matrix, and a deterministic filter enforces dietary
safety, which is needed because no finite penalty can certify foods near a
rule threshold. Shapley and Banzhaf attributions then explain each
allocation without changing it.

On NHANES-derived data, previous-day consumption carried most of the
predictive signal: the utility model ranked above five learned
recommenders but below a heuristic that repeats previous-day foods. On
held-out households of two to six members with exclusive allocation,
conflict resolution, the natural default, was the least equitable
strategy, while Nash allocation removed most of that inequity with higher
accuracy, came within a few percent of exact optima, and was EF1 in every
household; round robin performed almost as well.

Bargaining theory explains this result: the Nash rule is utilitarian and
egalitarian at once under natural weights and does not depend on how
members' utilities are scaled. For dishes cooked once for the household,
sharing saved about a third of the pantry, and a rotation of shared dishes
was usually the only weekly plan a finite pantry could supply. These
results come from households built from individual survey respondents,
with predicted rather than observed utilities and dietary rules that have
not been clinically reviewed; within that setting, they show that the
allocation rule shapes the balance between accuracy and equity as much as
the ranking model does.

\paragraph{\textbf{Future work.}} We see several directions for extending
this work. First, we will replace the fixed pantry of the main evaluation
with an inventory that tracks depletion, and extend rotation from shared
dishes to whole allocations, which realizes the convexified bargaining
problem of Proposition~\ref{prop:bargain}(b). Second, we will move
allocation within the safe sets so that verified bundles always hold $K$
items. Third, we will compute the egalitarian rule with an assignment
algorithm that carries a max-min guarantee and compare Nash allocation with
round robin across more settings. Finally, we plan a
blinded human evaluation of the explanations, a clinical review of the
dietary rules, and validation of the cohorts against observed households.

\section*{Acknowledgments}
The authors thank Vedant Balasubramaniam for his valuable
suggestions and insights, and Jonnalagadda Kowshik and Gannem Sai Shankar Narayana
for their contributions during the early stages of this work. AI-based language
models, including GPT and Claude, were used for language refinement and source
summarization. All AI-assisted content was reviewed and
verified by the authors for accuracy and factual integrity.

\bibliographystyle{abbrvnat}
\bibliography{references}

@inproceedings{zhang2025mopi,
  author    = {Zhang, Zheyuan and Wang, Zehong and Ma, Tianyi and Taneja, Varun Sameer and Nelson, Sofia and Le, Nhi Ha Lan and Murugesan, Keerthiram and Ju, Mingxuan and Chawla, Nitesh V. and Zhang, Chuxu and Ye, Yanfang},
  title     = {{MOPI-HFRS}: A Multi-Objective Personalized Health-Aware Food Recommendation System with {LLM}-Enhanced Interpretation},
  booktitle = {Proc. ACM SIGKDD (KDD)},
  year      = {2025},
  note      = {arXiv:2412.08847}
}

@misc{freitas2025swiss,
  author       = {{Abdur Rahman}, Lubnaa and Papathanail, Ioannis and Mougiakakou, Stavroula},
  title        = {Introducing the {Swiss} Food Knowledge Graph: {AI} for Context-Aware Nutrition Recommendation},
  howpublished = {arXiv:2507.10156},
  year         = {2025}
}

@inproceedings{Caldeira2018,
  author    = {Caldeira, Jefferson and Oliveira, Ricardo S. and Marinho, Leandro B. and Trattner, Christoph},
  title     = {Healthy Menus Recommendation: Optimizing the Use of the Pantry},
  booktitle = {Proc. HealthRecSys at RecSys},
  series    = {CEUR Workshop Proceedings},
  volume    = {2216},
  year      = {2018}
}

@inproceedings{Freyne2010,
  author    = {Freyne, Jill and Berkovsky, Shlomo},
  title     = {Intelligent Food Planning: Personalized Recipe Recommendation},
  booktitle = {Proc. ACM IUI},
  pages     = {321--324},
  year      = {2010}
}

@inproceedings{Ge2015,
  author    = {Ge, Mouzhi and Ricci, Francesco and Massimo, David},
  title     = {Health-Aware Food Recommender System},
  booktitle = {Proc. ACM RecSys},
  pages     = {333--334},
  year      = {2015}
}

@inproceedings{Trattner2017WWW,
  author    = {Trattner, Christoph and Elsweiler, David},
  title     = {Investigating the Healthiness of Internet-Sourced Recipes: Implications for Meal Planning and Recommender Systems},
  booktitle = {Proc. WWW},
  pages     = {489--498},
  year      = {2017}
}

@inproceedings{Haussmann2019,
  author    = {Haussmann, Steven and Seneviratne, Oshani and Chen, Yu and Ne'eman, Yarden and Codella, James and Chen, Ching-Hua and McGuinness, Deborah L. and Zaki, Mohammed J.},
  title     = {{FoodKG}: A Semantics-Driven Knowledge Graph for Food Recommendation},
  booktitle = {Proc. ISWC},
  pages     = {146--162},
  year      = {2019}
}

@article{Ma2024,
  author  = {Ma, Wenming and Li, Mingqi and Dai, Jian and Ding, Jianguo and Chu, Zihao and Chen, Hao},
  title   = {Nutrition-Related Knowledge Graph Neural Network for Food Recommendation},
  journal = {Foods},
  volume  = {13},
  number  = {13},
  pages   = {2144},
  year    = {2024}
}

@article{Deng2026,
  author  = {Deng, Xiwen and Tu, Wenzhe},
  title   = {Personalized Nutrition-Aware Dietary Recommendation with Multimodal Large Language Models},
  journal = {IEEE Access},
  volume  = {14},
  pages   = {21791--21804},
  year    = {2026}
}

@inproceedings{Burke2017,
  author    = {Burke, Robin},
  title     = {Multisided Fairness for Recommendation},
  booktitle = {Workshop on Fairness, Accountability and Transparency in Machine Learning (FAT/ML)},
  year      = {2017}
}

@article{Masthoff2004,
  author  = {Masthoff, Judith},
  title   = {Group Modeling: Selecting a Sequence of Television Items to Suit a Group of Viewers},
  journal = {User Modeling and User-Adapted Interaction},
  volume  = {14},
  number  = {1},
  pages   = {37--85},
  year    = {2004}
}

@inproceedings{Serbos2017,
  author    = {Serbos, Dimitris and Qi, Shuyao and Mamoulis, Nikos and Pitoura, Evaggelia and Tsaparas, Panayiotis},
  title     = {Fairness in Package-to-Group Recommendations},
  booktitle = {Proc. WWW},
  pages     = {371--379},
  year      = {2017}
}

@article{caragiannis2019unreasonable,
  author  = {Caragiannis, Ioannis and Kurokawa, David and Moulin, Herv{\'e} and Procaccia, Ariel D. and Shah, Nisarg and Wang, Junxing},
  title   = {The Unreasonable Fairness of Maximum {Nash} Welfare},
  journal = {ACM Transactions on Economics and Computation},
  volume  = {7},
  number  = {3},
  pages   = {1--32},
  year    = {2019}
}

@article{nguyen2014computing,
  author  = {Nguyen, Trung Thanh and Nguyen, Nam Dang and Roos, Magnus and Rothe, J{\"o}rg},
  title   = {Computational Complexity and Approximability of Social Welfare Optimization in Multiagent Resource Allocation},
  journal = {Autonomous Agents and Multi-Agent Systems},
  volume  = {28},
  number  = {2},
  pages   = {256--289},
  year    = {2014}
}

@inproceedings{Carraro2022AIxIA,
  author    = {Carraro, Tommaso and Daniele, Alessandro and Aiolli, Fabio and Serafini, Luciano},
  title     = {Logic Tensor Networks for {Top-N} Recommendation},
  booktitle = {AIxIA 2022 -- Advances in Artificial Intelligence},
  pages     = {110--123},
  publisher = {Springer},
  year      = {2022}
}

@article{Badreddine2022AI,
  author  = {Badreddine, Samy and d'Avila Garcez, Artur and Serafini, Luciano and Spranger, Michael},
  title   = {Logic Tensor Networks},
  journal = {Artificial Intelligence},
  volume  = {303},
  pages   = {103649},
  year    = {2022}
}

@inproceedings{Lundberg2017,
  author    = {Lundberg, Scott M. and Lee, Su-In},
  title     = {A Unified Approach to Interpreting Model Predictions},
  booktitle = {Advances in Neural Information Processing Systems (NeurIPS)},
  pages     = {4765--4774},
  year      = {2017}
}

@article{banzhaf1965weighted,
  author  = {Banzhaf, III, John F.},
  title   = {Weighted Voting Doesn't Work: A Mathematical Analysis},
  journal = {Rutgers Law Review},
  volume  = {19},
  pages   = {317--343},
  year    = {1965}
}

@inproceedings{Rendle2009,
  author    = {Rendle, Steffen and Freudenthaler, Christoph and Gantner, Zeno and Schmidt-Thieme, Lars},
  title     = {{BPR}: Bayesian Personalized Ranking from Implicit Feedback},
  booktitle = {Proc. UAI},
  pages     = {452--461},
  year      = {2009}
}

@inproceedings{He2017NeuMF,
  author    = {He, Xiangnan and Liao, Lizi and Zhang, Hanwang and Nie, Liqiang and Hu, Xia and Chua, Tat-Seng},
  title     = {Neural Collaborative Filtering},
  booktitle = {Proc. WWW},
  pages     = {173--182},
  year      = {2017}
}

@inproceedings{He2020LightGCN,
  author    = {He, Xiangnan and Deng, Kuan and Wang, Xiang and Li, Yan and Zhang, Yongdong and Wang, Meng},
  title     = {{LightGCN}: Simplifying and Powering Graph Convolution Network for Recommendation},
  booktitle = {Proc. ACM SIGIR},
  pages     = {639--648},
  year      = {2020}
}

@inproceedings{Hidasi2016GRU4Rec,
  author    = {Hidasi, Bal{\'a}zs and Karatzoglou, Alexandros and Baltrunas, Linas and Tikk, Domonkos},
  title     = {Session-Based Recommendations with Recurrent Neural Networks},
  booktitle = {Proc. ICLR},
  year      = {2016}
}

@inproceedings{Kang2018SASRec,
  author    = {Kang, Wang-Cheng and McAuley, Julian},
  title     = {Self-Attentive Sequential Recommendation},
  booktitle = {Proc. IEEE ICDM},
  pages     = {197--206},
  year      = {2018}
}

@incollection{fisher1978analysis,
  author    = {Fisher, Marshall L. and Nemhauser, George L. and Wolsey, Laurence A.},
  title     = {An Analysis of Approximations for Maximizing Submodular Set Functions---{II}},
  booktitle = {Polyhedral Combinatorics},
  series    = {Mathematical Programming Studies},
  volume    = {8},
  pages     = {73--87},
  publisher = {Springer},
  year      = {1978}
}

@article{korte1978analysis,
  author  = {Korte, Bernhard and Hausmann, Dirk},
  title   = {An Analysis of the Greedy Heuristic for Independence Systems},
  journal = {Annals of Discrete Mathematics},
  volume  = {2},
  pages   = {65--74},
  year    = {1978}
}

@article{lee2017apx,
  author  = {Lee, Euiwoong},
  title   = {{APX}-Hardness of Maximizing {Nash} Social Welfare with Indivisible Items},
  journal = {Information Processing Letters},
  volume  = {122},
  pages   = {17--20},
  year    = {2017}
}

@article{bezakova2005allocating,
  author  = {Bez{\'a}kov{\'a}, Ivona and Dani, Varsha},
  title   = {Allocating Indivisible Goods},
  journal = {ACM SIGecom Exchanges},
  volume  = {5},
  number  = {3},
  pages   = {11--18},
  year    = {2005}
}

@article{kaneko1979nash,
  author  = {Kaneko, Mamoru and Nakamura, Kenjiro},
  title   = {The {Nash} Social Welfare Function},
  journal = {Econometrica},
  volume  = {47},
  number  = {2},
  pages   = {423--435},
  year    = {1979}
}

@inproceedings{biswas2018fair,
  author    = {Biswas, Arpita and Barman, Siddharth},
  title     = {Fair Division Under Cardinality Constraints},
  booktitle = {Proc. IJCAI},
  pages     = {91--97},
  year      = {2018}
}

@inproceedings{bhatt2020evaluating,
  author    = {Bhatt, Umang and Weller, Adrian and Moura, Jos{\'e} M. F.},
  title     = {Evaluating and Aggregating Feature-Based Model Explanations},
  booktitle = {Proc. IJCAI},
  pages     = {3016--3022},
  year      = {2020}
}

@article{stelmakh2021peerreview4all,
  author  = {Stelmakh, Ivan and Shah, Nihar B. and Singh, Aarti},
  title   = {{PeerReview4All}: Fair and Accurate Reviewer Assignment in Peer Review},
  journal = {Journal of Machine Learning Research},
  volume  = {22},
  number  = {163},
  pages   = {1--66},
  year    = {2021}
}

@misc{mishra2023fair,
  author       = {Shaily Mishra and Manisha Padala and Sujit Gujar},
  title        = {Fair Allocation of Goods and Chores: Tutorial and Survey of Recent Results},
  howpublished = {arXiv:2307.10985},
  year         = {2023}
}

@inproceedings{conitzer2017fair,
  author    = {Vincent Conitzer and Rupert Freeman and Nisarg Shah},
  title     = {Fair Public Decision Making},
  booktitle = {Proceedings of the 2017 ACM Conference on Economics and Computation (EC)},
  pages     = {629--646},
  year      = {2017}
}

@inproceedings{fain2018fair,
  author    = {Brandon Fain and Kamesh Munagala and Nisarg Shah},
  title     = {Fair Allocation of Indivisible Public Goods},
  booktitle = {Proceedings of the 2018 ACM Conference on Economics and Computation (EC)},
  pages     = {575--592},
  year      = {2018}
}

@article{suksompong2021constraints,
  author  = {Warut Suksompong},
  title   = {Constraints in Fair Division},
  journal = {ACM SIGecom Exchanges},
  volume  = {19},
  number  = {2},
  pages   = {46--61},
  year    = {2021}
}

@book{narahari2014game,
  author    = {Y. Narahari},
  title     = {Game Theory and Mechanism Design},
  series    = {IISc Lecture Notes Series},
  volume    = {4},
  publisher = {World Scientific and IISc Press},
  address   = {Singapore},
  year      = {2014}
}

@article{nash1950bargaining,
  author  = {John F. Nash},
  title   = {The Bargaining Problem},
  journal = {Econometrica},
  volume  = {18},
  number  = {2},
  pages   = {155--162},
  year    = {1950}
}

@inproceedings{bien2020recipenlg,
  author    = {Micha{\l} Bie{\'n} and Micha{\l} Gilski and Martyna Maciejewska and Wojciech Taisner and Dawid Wisniewski and Agnieszka Lawrynowicz},
  title     = {{RecipeNLG}: A Cooking Recipes Dataset for Semi-Structured Text Generation},
  booktitle = {Proceedings of the 13th International Conference on Natural Language Generation (INLG)},
  pages     = {22--28},
  year      = {2020}
}

\appendix

\section{Derivations and Proofs}
\label{app:proofs}

Throughout, $\sigma(z)=1/(1+e^{-z})$ is the logistic function. Each proof
treats the parts of its proposition in order. Proposition~\ref{prop:greedy},
Proposition~\ref{prop:bargain}(a), and Proposition~\ref{prop:rotation}(a)
specialize known results to our setting, and we cite the source in each
case; the contribution there lies in the formulation and in the conditions
under which each result applies.

\subsection{The Soft Constraint Penalty}
\label{app:scp}

\begin{lemma}
\label{lem:softplus}
For every $\kappa>0$, the function $\rho_\kappa$ of Eq.~\ref{eq:scp} is
convex, increasing, and $1$-Lipschitz, and for all $z$,
\[
\max(0,z)\;\leq\;\rho_\kappa(z)\;\leq\;\max(0,z)+\frac{\ln 2}{\kappa}.
\]
\end{lemma}
\begin{proof}
The first two derivatives are
\[
\rho_\kappa'(z)=\sigma(\kappa z)\in(0,1),
\qquad
\rho_\kappa''(z)=\kappa\,\sigma(\kappa z)\big(1-\sigma(\kappa z)\big)>0,
\]
which give monotonicity, convexity, and the Lipschitz bound. For the
sandwich bound,
\[
\rho_\kappa(z)-\max(0,z)=\frac{1}{\kappa}\log\big(1+e^{-\kappa|z|}\big)\in\Big(0,\frac{\ln 2}{\kappa}\Big].
\qedhere
\]
\end{proof}

\begin{lemma}[Probabilistic reading of the penalty]
\label{lem:prob}
Suppose the true slack of rule $c$ is $g_c(f)+\xi_c$, where the $\xi_c$ are
independent logistic variables with location $0$ and scale $1/\kappa$, and
let $w_c=1$. Then
\[
-\kappa\,\SCP(u_i,f)=\sum_{c\in\cC_i}\log\sigma\big(-\kappa g_c(f)\big)=\log\Pr\big[f\text{ satisfies every rule in }\cC_i\big].
\]
Moreover, with $\lambda=\kappa$, the utility of Eq.~\ref{eq:utility}
satisfies
\[
U_i(f)\;\geq\;\sigma\big(\ell_i(f)\big)\cdot\Pr[f\text{ is safe for }i],
\]
with near equality when relevance probabilities are small.
\end{lemma}
\begin{proof}
The logistic distribution with scale $1/\kappa$ has distribution function
$z\mapsto\sigma(\kappa z)$, so each rule holds with probability
$\Pr[\xi_c\leq-g_c(f)]=\sigma(-\kappa g_c(f))$. Since
$\log\sigma(-x)=-\log(1+e^{x})$,
\[
\log\sigma\big(-\kappa g_c(f)\big)=-\kappa\,\rho_\kappa\big(g_c(f)\big),
\]
and summing over the independent rules gives the identity.

For the bound, write $b=\lambda\,\SCP(u_i,f)\geq0$, so that
$U_i(f)=\sigma(\ell_i(f)-b)$ and $e^{-b}=\Pr[f\text{ is safe for }i]$. For
all $\ell\in\mathbb{R}$,
\[
\frac{\sigma(\ell-b)}{\sigma(\ell)\,e^{-b}}=\frac{e^{b}+e^{b-\ell}}{1+e^{b-\ell}}\;\geq\;1,
\]
because $e^{b}\geq1$, and the ratio tends to $1$ as $\ell\to-\infty$.
\end{proof}

If relevance and safety are treated as independent events, the right-hand
side $\sigma(\ell_i(f))\cdot\Pr[f\text{ is safe for }i]$ is the probability
that the food is both relevant and safe, which is the reading given in
Section~\ref{sec:utility}.

\subsection{Proof of Proposition~\ref{prop:scp}}
Since $\sigma$ is increasing, ranking by $U_i$ is ranking by the penalized
logit
\[
s_i(f)=\ell_i(f)-\lambda\,\SCP(u_i,f).
\]

\paragraph{Part (a).} Let $f^+$ be safe, and let $f^-$ violate some rule $c'$
by at least $\gamma$. Every slack of $f^+$ is non-positive, so
Lemma~\ref{lem:softplus} gives an upper bound on its penalty. Every term of
the penalty is positive, so the violated rule alone gives a lower bound on
the penalty of $f^-$:
\[
\SCP(u_i,f^+)\leq\frac{\ln 2}{\kappa}\sum_{c}w_c,
\qquad
\SCP(u_i,f^-)\geq w_{c'}\,\rho_\kappa(\gamma)\geq w_{\min}\,\rho_\kappa(\gamma).
\]
Since $\ell_i(f^+)-\ell_i(f^-)\geq-2L$,
\[
s_i(f^+)-s_i(f^-)\;\geq\;-2L+\lambda\Big(w_{\min}\,\rho_\kappa(\gamma)-\frac{\ln 2}{\kappa}\sum_{c}w_c\Big)\;>\;0
\]
under the condition of the proposition. As $\gamma\to0$,
$\rho_\kappa(\gamma)\to\kappa^{-1}\ln 2$, so the bracket tends to
\[
\frac{\ln 2}{\kappa}\Big(w_{\min}-\sum_c w_c\Big)\leq0,
\]
and no finite $\lambda$ satisfies the condition.

\paragraph{Part (b).} The two foods differ only in rule $c$, so by the
Lipschitz bound of Lemma~\ref{lem:softplus},
\[
\SCP(u_i,f^-)-\SCP(u_i,f^+)=w_c\big(\rho_\kappa(\eta)-\rho_\kappa(0)\big)\leq w_c\,\eta.
\]
Hence, if $\ell_i(f^-)-\ell_i(f^+)>\lambda w_c\eta$, then
$s_i(f^-)>s_i(f^+)$, and the unsafe $f^-$ ranks above $f^+$. \qed

\subsection{Proof of Proposition~\ref{prop:greedy}}
Let $\Omega=N\times\cF$, and define
\[
\mathcal{I}_1=\big\{X\subseteq\Omega:\ |X\cap(N\times\{f\})|\leq1\ \text{for all } f\big\},
\qquad
\mathcal{I}_2=\big\{X\subseteq\Omega:\ |X\cap(\{i\}\times\cF)|\leq K\ \text{for all } i\big\}.
\]
$\mathcal{I}_1$ is a partition matroid with one block per food,
$\mathcal{I}_2$ a partition matroid with one block per member, and the
feasible allocations are exactly $\mathcal{I}_1\cap\mathcal{I}_2$.

\paragraph{Part (a).} Write
\[
\Phi_\epsilon(X)=\sum_i h\Big(\sum_{f:(i,f)\in X}U_i(f)\Big),
\qquad h(t)=\log(1+t/\epsilon).
\]
The function $h$ is concave and increasing with $h(0)=0$. A concave,
non-decreasing function of a non-negative additive set function is
monotone and submodular, and so is a sum of such functions. Hence
$\Phi_\epsilon$ is normalized ($\Phi_\epsilon(\varnothing)=0$), monotone,
and submodular. The marginal gain of adding $(i,f)$ is
\[
\log\Big(1+\frac{U_i(f)}{\epsilon+U_i(A_i)}\Big)>0,
\]
so Algorithm~\ref{alg:greedy} with the Nash rule adds, at each step, the
feasible pair of largest gain, until no pair can be added. Fisher et
al.~\cite{fisher1978analysis} show that this greedy algorithm attains at
least $1/(p+1)$ of the optimum of a normalized monotone submodular function
over the intersection of $p$ matroids. Here $p=2$.

\paragraph{Part (b).} $W_{\mathrm{UT}}$ is additive with non-negative
weights, and the utilitarian rule adds the feasible pair of largest weight.
Over the intersection of $p$ matroids, this greedy algorithm attains at
least $1/p$ of the optimum~\cite{korte1978analysis}; here $p=2$. The exact
optimum assigns each food to at most one member and at most $K$ foods to
each member at maximum total weight. This is a maximum-weight bipartite
$b$-matching, which minimum-cost flow computes in polynomial time.

\paragraph{Part (c).} Setting $K\geq m$ recovers the unconstrained max-min
allocation problem, which is NP-hard to approximate within any factor below
two~\cite{bezakova2005allocating}. \qed

\subsection{Proof of Proposition~\ref{prop:bargain}}

\paragraph{Part (a).} For constants $a_i>0$,
\[
\prod_i a_iU_i(A_i)=\Big(\prod_i a_i\Big)\prod_i U_i(A_i),
\]
so the Nash product changes by a positive constant and its maximizers are
unchanged. For the other two rules, two-member examples suffice; in each,
the optimum is unique.
\begin{itemize}[leftmargin=1.5em,itemsep=2pt,topsep=3pt]
\item \emph{Utilitarian.} Let $K=1$, $U_1=(0.1,0.3)$, and $U_2=(0.2,0.5)$
over two foods. The exact and greedy rules both give member~1 the first
food, worth $0.1$. After member~1's utilities are doubled, both give it the
second food, worth $0.3$ to it.
\item \emph{Egalitarian.} Let $K=2$, $U_1=(0.5,0.4,0.3)$, and
$U_2=(0.5,0.45,0.02)$ over three foods. The exact optimum gives member~1
utility $0.7$, and the greedy rule gives it $0.5$. After member~1's
utilities are halved, both give it the first and third foods, worth $0.8$.
\end{itemize}

\paragraph{Part (b).} Let $G(x)=\sum_i\log x_i$. The set $\mathrm{conv}(F)$
is compact and convex and contains a point with positive entries, and $G$
is strictly concave, so $G$ has a unique maximizer $x^\star>0$ over
$\mathrm{conv}(F)$, which also maximizes the Nash product.

\emph{Necessity.} For $y\in\mathrm{conv}(F)$ and $\tau\in(0,1]$, the point
$x^\star+\tau(y-x^\star)$ lies in $\mathrm{conv}(F)$. Optimality of
$x^\star$ gives a non-positive directional derivative,
\[
\sum_i\frac{y_i-x^\star_i}{x^\star_i}\leq0.
\]
With $\lambda_i=1/x^\star_i$, this reads
$\sum_i\lambda_iy_i\leq\sum_i\lambda_ix^\star_i$, and $\lambda_ix^\star_i=1$
for every $i$.

\emph{Sufficiency.} Suppose $x^\star>0$ maximizes $\sum_i\lambda_ix_i$ over
$\mathrm{conv}(F)$ and $\lambda_ix^\star_i=c$ for all $i$. Then
$\lambda_i=c/x^\star_i$, so $\sum_iy_i/x^\star_i\leq n$ for every
$y\in\mathrm{conv}(F)$. For $y>0$, concavity of $G$ gives
\[
G(y)\;\leq\;G(x^\star)+\sum_i\frac{y_i}{x^\star_i}-n\;\leq\;G(x^\star),
\]
and $G(y)=-\infty$ if some $y_i=0$. Hence $x^\star$ is the Nash solution.
For $n=2$, this is Proposition~26.1 of~\cite{narahari2014game}.

\paragraph{Part (c).} With $\lambda_i=1/(\epsilon+U_i(A_i))$, the gain of
giving $f$ to member $i$ is $\log\big(1+\lambda_iU_i(f)\big)$. It increases
in $\lambda_iU_i(f)$, and each member's best remaining food $f_i^\star$ does
not depend on $\lambda_i$, so the Nash rule selects
$\arg\max_i\lambda_iU_i(f_i^\star)$. These are the weights of~(b) at the
point $x_i=U_i(A_i)$ for disagreement point $-\epsilon$.

Rescaling member $i$'s utilities by $a_i>0$ leaves $f_i^\star$ unchanged and
turns its score into
\[
\frac{a_iU_i(f_i^\star)}{\epsilon+a_iU_i(A_i)}
\;\longrightarrow\;
\begin{cases}
U_i(f_i^\star)/U_i(A_i) & \text{if } A_i\neq\varnothing,\\[2pt]
a_iU_i(f_i^\star)/\epsilon & \text{if } A_i=\varnothing,
\end{cases}
\qquad\text{as }\epsilon\to0.
\]
The first limit does not depend on $a_i$. The second exceeds the score of
every member holding an item for small $\epsilon$, so members with empty
bundles pick first, in decreasing order of $a_iU_i(f_i^\star)$, and only
this order depends on the scaling. When it is fixed and there are no ties,
induction over the finitely many steps shows that, for sufficiently small
$\epsilon$, every scaling produces the same allocation. \qed

\subsection{Proof of Proposition~\ref{prop:placement}}
Filtering only removes foods, so the bundles of $\hat{A}^{\mathrm{all}}$
remain disjoint, hold at most $K$ items, and satisfy
$\hat{A}^{\mathrm{all}}_i=A^{\mathrm{all}}_i\cap S_i\subseteq S_i$. Hence
$\hat{A}^{\mathrm{all}}$ is feasible for the problem that
$A^{\mathrm{safe}}$ solves, and
$W(A^{\mathrm{safe}})\geq W(\hat{A}^{\mathrm{all}})$. The argument requires
exact maximizers; with greedy allocation, the inequality is expected but not
guaranteed instance by instance. \qed

\subsection{Proof of Proposition~\ref{prop:rotation}}

\paragraph{Part (a).} Each pure strategy is a vertex of $\Delta(\Pi)$, so
$v^\star\geq v_{\mathrm{pure}}$. By the minimax
theorem~\cite[Thm.~9.2]{narahari2014game},
\[
\max_{p\in\Delta(\Pi)}\,\min_{w\in\Delta(N)}\,p^{\top}Mw
\;=\;
\min_{w\in\Delta(N)}\,\max_{p\in\Delta(\Pi)}\,p^{\top}Mw .
\]
For fixed $p$, the linear function $w\mapsto p^{\top}Mw$ attains its
minimum over the simplex at a vertex, which gives
$\min_i\sum_\pi p_\pi M_{\pi i}$; likewise, the inner maximum on the right
is $\max_\pi\sum_iw_iM_{\pi i}$. The planner's problem and its
dual~\cite[Sec.~9.4]{narahari2014game} are the linear programs
\[
\begin{array}{llll}
\max & t & \qquad\min & s\\[2pt]
\text{s.t.} & \sum_{\pi}p_\pi M_{\pi i}\geq t\quad (i\in N), & \qquad\text{s.t.} & \sum_i w_iM_{\pi i}\leq s\quad (\pi\in\Pi),\\[2pt]
 & \sum_\pi p_\pi=1,\ \ p\geq0, & & \sum_i w_i=1,\ \ w\geq0.
\end{array}
\]

\paragraph{Part (b).} Utilities are non-negative, so $v^\star\geq0$. If
$v^\star=0$, every single pair is optimal. Otherwise the constraint $t\geq0$
does not change the optimum, and adding slack variables $s_i\geq0$ gives a
program with $n+1$ equality constraints and non-negative variables. The
program is feasible and bounded, so it has an optimal basic feasible
solution with at most $n+1$ non-zero variables. One of them is
$t=v^\star>0$, so at most $n$ components of $p$ are non-zero. With stock
constraints, the slack of every non-binding stock constraint is also
non-zero, so the support is at most $n$ plus the number of binding stock
constraints.

\paragraph{Part (c).} This is complementary slackness for the primal--dual
pair in~(a). If $w^\star_i>0$, member $i$'s constraint binds,
$\sum_\pi p^\star_\pi M_{\pi i}=v^\star$. If $p^\star_\pi>0$, pair $\pi$'s
dual constraint binds,
\[
\sum_iw^\star_iM_{\pi i}=v^\star=\max_{\pi'\in\Pi}\,\sum_iw^\star_iM_{\pi' i}.
\qquad\qed
\]

\section{Additional Results}
\label{app:results}

\paragraph{Implementation details.} The utility model uses 384-dimensional all-MiniLM-L6-v2 sentence embeddings for foods and
evidence, a 128-dimensional latent space, a two-layer transformer encoder
(4 heads) with attention pooling for the retrieved evidence, and 4-head
attention for fusion. It is trained with binary cross-entropy on the
training users' second-day foods as positives, with four negatives per
positive sampled in proportion to first-day popularity raised to 0.75,
using Adam (learning rate $10^{-3}$) for 3 epochs with batch size 256; the
last epoch is kept, without a validation set. User embeddings are TF-IDF
weighted means of first-day food embeddings, with inverse document
frequencies from the training users. The retrieval corpus consists of six
dietary guideline statements, indexed by cosine similarity with top-$k=6$,
and each condition has a fixed health query. The baselines are in-house
implementations with one setting each and no tuning: 64-dimensional
embeddings, learning rate 0.005, 15 epochs, BPR with one uniform negative
for BPRMF, NeuMF, and LightGCN, and cross-entropy for GRU4Rec and SASRec.
SASRec returns the same rankings for every seed, so its training appears
not to change its scores.

\paragraph{Exact optima and EF1.} Table~\ref{tab:exact} compares greedy
allocation with exact optima computed by integer programming on the union
of the members' top 30 safe foods, for 100 held-out households per size;
every instance solves to proven optimality. For the Nash rule we compare the
smoothed objective with an upper bound from a piecewise-linear relaxation of
the logarithm. Greedy allocation attains 95--99\% of the optimum, far above
the worst-case bounds of Proposition~\ref{prop:greedy}. The utilitarian,
egalitarian, and Nash allocations and round robin are EF1 under safe-part
valuations in every household; conflict resolution is EF1 in 99\%, 98\%,
and 92\% of households for $n=2$, 3, and 4. On these four-person
households, round robin reaches a Recall@20 of 0.216, Gini 0.097, and NSW
10.47, against 0.218, 0.106, and 10.63 for Nash allocation.

\begin{table}[htbp]
\caption{Greedy allocation relative to exact optima (mean over 100 held-out
households per size, with 95\% bootstrap intervals).}
\label{tab:exact}
\centering
\small
\begin{tabular}{cccc}
\toprule
$n$ & \textbf{Utilitarian: greedy/opt.} & \textbf{Max-min: greedy/opt.} & \textbf{Nash: $\Phi_\epsilon$/upper bound} \\
\midrule
2 & 0.986 [0.982, 0.990] & 0.978 [0.973, 0.983] & 0.996 [0.995, 0.997] \\
3 & 0.975 [0.971, 0.980] & 0.957 [0.950, 0.964] & 0.990 [0.988, 0.992] \\
4 & 0.976 [0.972, 0.980] & 0.946 [0.937, 0.955] & 0.987 [0.985, 0.989] \\
\bottomrule
\end{tabular}
\end{table}

\paragraph{Shared-meal settings.} The pantry records
ingredient quantities: a stock vector $q\geq0$ over ingredients, and for
each dish $f$ a requirement vector $r(f)$, the quantity of each ingredient
needed to prepare $f$ as allocated. For an individual item this is one
serving; for a shared dish it is one batch cooked for the whole household.
A dish is in stock if $r(f)\leq q$ componentwise, and allocating it sets
$q\gets q-r(f)$. We evaluate the extension in two settings, each with 100
four-person households under each of five seeds, $K=6$, and
$\epsilon=10^{-5}$.

\emph{Indian recipes, simulated utilities.} A corpus of 3,597 Indian recipes
(911 grains and breads, 510 legumes and dals, and 2,176 sides and curries)
selected from the RecipeNLG corpus~\cite{bien2020recipenlg}. Each member is vegetarian with
probability 0.35, avoids nuts with probability 0.08, and avoids dairy with
probability 0.10; a dish violates a rule when its title or ingredient list
matches the excluded category. The corpus has no interaction histories, so
the utility model of Section~\ref{sec:utility} cannot be trained on it;
member utilities are simulated from random ingredient-affinity vectors,
correlated across the members of a household with coefficient 0.5. The
pantry stocks the 200 most frequent ingredients of the corpus with 2 to 7
units each, and a dish uses one unit of each of its ingredients.

\emph{NHANES foods, learned utilities.} The households, utilities, and
dietary rules of the main evaluation, so utilities come from the trained
utility model. Keyword matching on food descriptions identifies 171 grains
and breads and 114 legumes among the 8,704 foods. NHANES foods have no
ingredient lists, so we approximate ingredients by the 100 most frequent
words in food descriptions (2.9 per food on average) and stock each with 2
to 7 units.

In both settings, the rules are expressed as slacks $g_c(f)$ as in
Section~\ref{sec:problem}, so a dish is safe for a member exactly when it
would pass the filter of Section~\ref{sec:safety}.

\begin{algorithm}[t]
\caption{Two-stage shared-meal allocation}
\label{alg:shared_meal}
\begin{algorithmic}[1]
\Require utilities $U_i(f)$; safe sets $S_i$; stock $q$; requirements $r(\cdot)$; cap $K\geq2$; $\epsilon>0$
\Statex \textbf{Stage 1: max-min shared dishes}
\State $\mathcal{G}\gets$ grains in $\bigcap_i S_i$; \ $\mathcal{L}\gets$ dals in $\bigcap_i S_i$
\State $\Pi\gets\{(g,d)\in\mathcal{G}\times\mathcal{L}: r(g)+r(d)\leq q\}$ \Comment{componentwise}
\If{$\Pi\neq\varnothing$}
  \State $(g^\star,d^\star)\gets\arg\max_{(g,d)\in\Pi}\ \min_{i}\big(U_i(g)+U_i(d)\big)$ \Comment{exhaustive; ties: lowest index}
  \State $A_i\gets\{g^\star,d^\star\}$ for every $i$; \ $q\gets q-r(g^\star)-r(d^\star)$; \ $F_{\mathrm{rem}}\gets\cF\setminus\{g^\star,d^\star\}$
\Else
  \State $A_i\gets\varnothing$ for every $i$; \ $F_{\mathrm{rem}}\gets\cF$ \Comment{no safe shared meal exists}
\EndIf
\Statex \textbf{Stage 2: Nash-greedy individual slots}
\State $C_i\gets\{f\in F_{\mathrm{rem}}\cap S_i: r(f)\leq q\}$ for every $i$; \ $I\gets\{i:|A_i|<K,\ C_i\neq\varnothing\}$
\While{$I\neq\varnothing$}
  \State $f_i^\star\gets\arg\max_{f\in C_i}U_i(f)$ for every $i\in I$
  \State $i^\star\gets\arg\max_{i\in I}U_i(f_i^\star)/\big(\epsilon+U_i(A_i)\big)$ \Comment{ties: lowest index}
  \State $A_{i^\star}\gets A_{i^\star}\cup\{f^\star_{i^\star}\}$; \ $F_{\mathrm{rem}}\gets F_{\mathrm{rem}}\setminus\{f^\star_{i^\star}\}$; \ $q\gets q-r(f^\star_{i^\star})$
  \State recompute $C_i$ for every $i$ and $I$
\EndWhile
\State \Return $A$
\end{algorithmic}
\end{algorithm}

Table~\ref{tab:scaling} reports all metrics of the household-size sweep
plotted in Fig.~\ref{fig:scaling}, and Table~\ref{tab:worked_example} gives
the household-level example referred to in Section~\ref{sec:shared_meals}.

\begin{table}[htbp]
\caption{Household allocation across household sizes $n=2$--$6$ (100
held-out households per size). CSR is 1.00 in every row. Items is the mean
number of foods each member receives; strategies A and C--E always fill all
$K=20$ slots. Bootstrap confidence intervals are released with the code.}
\label{tab:scaling}
\centering
\setlength{\tabcolsep}{4.5pt}
\begin{tabular}{clccccccc}
\toprule
\textbf{$n$} & \textbf{Strategy} & \textbf{Recall@20} & \textbf{NDCG@20} & \textbf{PUR} & \textbf{Min Util.} & \textbf{Gini} & \textbf{NSW} & \textbf{Items} \\
\midrule
2 & A) Independent & 0.292 & 0.292 & 0.187 & 11.37 & 0.072 & 12.96 & 20.0 \\
2 & B) Conflict resolution & 0.252 & 0.263 & 0.142 & 8.92 & 0.105 & 10.82 & 15.1 \\
2 & C) Utilitarian & 0.264 & 0.270 & 0.184 & 10.34 & 0.086 & 12.16 & 20.0 \\
2 & D) Max-Min & 0.260 & 0.268 & 0.185 & 10.84 & 0.061 & 12.13 & 20.0 \\
2 & E) Nash & 0.267 & 0.273 & 0.185 & 10.75 & 0.072 & 12.30 & 20.0 \\
\midrule
3 & A) Independent & 0.283 & 0.291 & 0.279 & 10.36 & 0.096 & 12.87 & 20.0 \\
3 & B) Conflict resolution & 0.203 & 0.232 & 0.182 & 6.81 & 0.152 & 9.47 & 13.1 \\
3 & C) Utilitarian & 0.221 & 0.242 & 0.270 & 8.64 & 0.122 & 11.33 & 20.0 \\
3 & D) Max-Min & 0.215 & 0.242 & 0.270 & 9.50 & 0.078 & 11.22 & 20.0 \\
3 & E) Nash & 0.219 & 0.246 & 0.271 & 9.36 & 0.097 & 11.52 & 20.0 \\
\midrule
4 & A) Independent & 0.299 & 0.308 & 0.370 & 9.33 & 0.107 & 12.40 & 20.0 \\
4 & B) Conflict resolution & 0.206 & 0.237 & 0.215 & 5.17 & 0.187 & 8.31 & 11.7 \\
4 & C) Utilitarian & 0.219 & 0.245 & 0.355 & 7.38 & 0.140 & 10.43 & 20.0 \\
4 & D) Max-Min & 0.207 & 0.239 & 0.358 & 8.42 & 0.082 & 10.23 & 20.0 \\
4 & E) Nash & 0.218 & 0.248 & 0.359 & 8.22 & 0.106 & 10.63 & 20.0 \\
\midrule
5 & A) Independent & 0.274 & 0.283 & 0.466 & 9.05 & 0.116 & 12.62 & 20.0 \\
5 & B) Conflict resolution & 0.171 & 0.204 & 0.250 & 4.18 & 0.217 & 7.73 & 10.8 \\
5 & C) Utilitarian & 0.184 & 0.212 & 0.435 & 6.64 & 0.160 & 10.10 & 20.0 \\
5 & D) Max-Min & 0.173 & 0.207 & 0.437 & 8.07 & 0.088 & 9.94 & 20.0 \\
5 & E) Nash & 0.184 & 0.216 & 0.438 & 7.84 & 0.117 & 10.42 & 20.0 \\
\midrule
6 & A) Independent & 0.282 & 0.288 & 0.563 & 8.69 & 0.117 & 12.60 & 20.0 \\
6 & B) Conflict resolution & 0.162 & 0.195 & 0.271 & 3.49 & 0.233 & 6.97 & 9.7 \\
6 & C) Utilitarian & 0.171 & 0.201 & 0.508 & 6.00 & 0.174 & 9.51 & 20.0 \\
6 & D) Max-Min & 0.160 & 0.195 & 0.512 & 7.50 & 0.082 & 9.25 & 20.0 \\
6 & E) Nash & 0.171 & 0.204 & 0.513 & 7.25 & 0.120 & 9.86 & 20.0 \\
\bottomrule
\end{tabular}
\end{table}

Table~\ref{tab:worked_example} shows how the two allocations differ for one
held-out NHANES household with learned utilities. Member~4 has hypertension
and the others have no condition. Under single servings, no member receives
both a grain and a legume; member~4 receives toasted wheat bread but no
legume. Structured shared allocation cooks white bread and pinto beans once,
both safe for every member including member~4's sodium limit, so every
member gets a complete meal and pantry use falls from 14.35\% to 10.71\%.
Each member's four individual foods are exactly their first four
single-serving picks, so the shared dishes replace each member's two
lowest-ranked foods. Members value the shared dishes at 0.14--0.23. Members
1, 2, and 4 lose 1.4--1.6 in utility, while member~3, whose two
lowest-ranked foods were worth little (0.74 and 0.28), loses 0.65, so the
spread of totals narrows from 4.74--5.85 to 4.09--4.33.

\begin{table}[htbp]
\caption{Worked example: one held-out NHANES household with learned
utilities, $K=6$. Single-serving allocation gives each member six foods;
structured shared allocation cooks white bread and pinto beans once for all
members and adds four individual foods. Totals are bundle utilities on the
min-max-normalized scale. Both shared dishes are safe for every member.}
\label{tab:worked_example}
\centering
\small
\begin{tabular}{@{}l >{\raggedright\arraybackslash}p{5.0cm} c >{\raggedright\arraybackslash}p{4.7cm} c@{}}
\toprule
\textbf{Member} & \textbf{Single-serving foods} & \textbf{Total} & \textbf{Shared dishes (utility) + individual foods} & \textbf{Total} \\
\midrule
1 (none) & Baby oatmeal cereal, sweetened tea, fruit yogurt, coffee cake, barbecue chicken, corn puffs & 5.84 & Bread 0.20, beans 0.14; oatmeal cereal, tea, fruit yogurt, coffee cake & 4.28 \\
2 (none) & Froot Loops, yeast doughnut, mozzarella, hamburger, cola, chocolate sandwich cookie & 5.57 & Bread 0.23, beans 0.16; Froot Loops, doughnut, mozzarella, hamburger & 4.19 \\
3 (none) & Low-fat fruit yogurt, whole milk, salted pretzels, cheese pizza, French fries, banana & 4.74 & Bread 0.22, beans 0.16; yogurt, whole milk, pretzels, cheese pizza & 4.09 \\
4 (hypertension) & Bacon, fried eggs, French dressing, tap water, toasted wheat bread, heavy cream & 5.85 & Bread 0.23, beans 0.17; bacon, fried eggs, French dressing, tap water & 4.33 \\
\midrule
\multicolumn{5}{@{}l}{Complete meals: 0 of 4 (single-serving) vs.\ 4 of 4 (shared). Pantry used: 14.35\% vs.\ 10.71\%.} \\
\bottomrule
\end{tabular}
\end{table}

\end{document}